\documentclass[final,12pt]{colt2019} 

\usepackage{subcaption}
\usepackage{graphicx}

\def\colorful{0}

\usepackage{amsfonts,amsmath,amssymb,color,float,graphicx,verbatim}
\usepackage{multirow}

\newif\ifhyper\IfFileExists{hyperref.sty}{\hypertrue}{\hyperfalse}
\hyperfalse
\hypertrue
\ifhyper\usepackage{hyperref}\fi

\usepackage{enumitem}
\usepackage[capitalize]{cleveref} 
\usepackage{algorithm}

\def\nnewcolor{1}
\ifnum\nnewcolor=1

\fi
\ifnum\nnewcolor=0

\fi

\ifnum\colorful=1
\newcommand{\new}[1]{{\color{red} #1}}

\else
\newcommand{\new}[1]{{#1}}

\fi
\newcommand{\hide}[1]{}

\usepackage{multirow}
\usepackage{amsmath, color, enumitem}

 \usepackage{framed}

\usepackage{nicefrac}

\newtheorem{claim}[theorem]{Claim}
\newtheorem{fact}[theorem]{Fact}

\newcommand{\E}{\mathbf{\mathbf{E}}}
\newcommand{\Var}{\mathbf{Var}}
\newcommand{\Prob}{\mathbf{Pr}}

\newcommand{\dtv}{d_{\mathrm TV}}
\newcommand{\wh}[1]{{\widehat{#1}}}

\newcommand{\ignore}[1]{}

\newcommand{\eps}{\epsilon}

\newcommand{\eqdef}{\stackrel{{\mathrm {\footnotesize def}}}{=}}

\newcommand{\littlesum}{\mathop{\textstyle \sum}}

\title[Communication and Memory Efficient Testing of Discrete Distributions]{Communication and Memory Efficient Testing \\ of Discrete Distributions}

\coltauthor{%
 \Name{Ilias Diakonikolas} \Email{\tt diakonik@usc.edu}\\
 \addr University of Southern California
 \AND
 \Name{Themis Gouleakis} \Email{\tt tgouleak@mpi-inf.mpg.de}\\
 \addr Max Planck Institute for Informatics
 \AND 
 \Name{Daniel M. Kane} \Email{\tt dakane@cs.ucsd.edu}\\
 \addr University of California, San Diego
 \AND
 \Name{Sankeerth Rao} \Email{\tt sankeerth1729@gmail.com}\\
 \addr University of California, San Diego
}

\begin{document}

\maketitle

\begin{abstract}
We study distribution testing with communication and memory constraints 
in the following computational models: (1) The {\em one-pass streaming model} 
where the goal is to minimize the sample complexity of the protocol subject to a memory constraint,
and (2) A {\em distributed model} where the data samples reside at multiple machines and the
goal is to minimize the communication cost of the protocol. In both these models, we provide efficient
algorithms for uniformity/identity testing (goodness of fit) and closeness testing (two sample testing).
Moreover, we show nearly-tight lower bounds on (1) the sample complexity
of any one-pass streaming tester for uniformity, subject to the memory constraint, 
and (2) the communication cost of any uniformity testing protocol, 
in a restricted ``one-pass'' model of communication.
\end{abstract}


\section{Introduction} \label{sec:intro}

\subsection{Background} \label{ssec:background}

Classical statistics theory focuses on characterizing
the inherent sample complexity of inference tasks, typically formalized
via minimax rates of convergence. Research in this field has primarily focused on understanding
the sample complexity of inference in the centralized setting, where all the samples are available
to a single machine that performs the computation. We now have a rich theory
(see, e.g.,~\cite{DG85, DL:01, Tsybakov08} for a few books on the topic) that has led to
characterizing the sample complexity of a wide range of statistical tasks in this regime.

In modern data analysis, one may have additional constraints on data collection
and storage. Modern datasets are often too large to be stored on a single computer,
and so it is natural to consider methods that either impose upper bounds on the available memory
or involve multiple machines, each containing a small subset of the dataset.
Typical examples include anomaly detection in various settings (e.g., inference based on distributed sensor measurements,
fraud detection based on different transactions of a customer, deciding whether a region of the sky is interesting
based on astronomical data from multiple telescopes, etc.)

In this paper, we study {\em distribution property testing}~\cite{BFR+:00} in the following computational
models: (1) The {\em one-pass streaming model} where the goal is to minimize the sample complexity
of the protocol subject to a memory constraint,
and (2) A {\em distributed model} where the data samples reside at multiple machines, and the
goal is to minimize the communication cost of the protocol. In both these models, we provide efficient
algorithms for uniformity/identity testing (goodness of fit) and closeness testing (two sample testing).
Moreover, we show lower bounds (in some cases, nearly-tight) on (1) the sample complexity
of any one-pass streaming tester, subject to the memory constraint and (2) the communication cost
of any protocol performing the testing task (in a restricted ``one-pass'' model of communication, described below).

\paragraph{Computational Models}
In the {\em one-pass streaming model}, the data samples are revealed online in a stream and the algorithm is allowed
a single pass over the data. Moreover, there is an upper bound, which we will typically denote by $m$,
on the number of bits the algorithm can store at any point of its execution. In our setting, the goal is to minimize the sample complexity of testing subject to the memory constraint.

Our {\em distributed communication model}
uses a blackboard (broadcast) model of communication
in the sense that each message sent by each machine (player)
is visible to all machines. There is an arbitrarily large number of machines,
each holding $\ell$ independent samples from the unknown distribution(s).
Additionally, there is a referee (arbitrator) who holds no samples.
In each round, the referee either returns an answer or asks \new{a one-bit}
question to one of the players about their input and receives a
response. The goal is for the referee to return the correct answer
to our testing problem (with at least $2/3$ probability) in as few rounds
of communication as possible. Notice that this model only costs the
communication needed to answer the referee's questions and not the
information encoded by the questions themselves or by which player the
referee chooses to ask. This is natural in a broadcast communication model,
as this information would be implicitly determined
by the communication transcript up to this point.

Unfortunately, we do not know how to prove lower bounds in
the above general model, and will instead work in the {\em one-pass version} of
this model. In the one-pass version, the referee is
not allowed to go back to querying a player after they have moved
on. In particular, the referee cannot ask a question to player
$A$, subsequently ask a question to player $B \neq A$, and then ask
a question to player $A$ again. Our communication lower bounds hold
in this one-pass model. We note that our algorithms work
in the one-pass model as well.

\subsection{Our Contributions}\label{ssec:contribution}
We give algorithms and lower bounds for uniformity/identity 
testing\footnote{The reduction of identity to uniformity in~\cite{Goldreich16} immediately 
translates our upper bounds from uniformity to identity.}
and closeness testing in the presence of communication and memory constraints.
More specifically, we obtain the following results:

\begin{enumerate}
\item A one-pass streaming algorithm for uniformity testing on $[n]$ with memory upper bound of $m$ bits 
        that has sample complexity $O(n\log(n)/(m \eps^4))$ for a broad range of parameters. Moreover, we show that 
        this sample upper bound is tight for a fairly wide range of $m$, and tight up to a $\log(n)$ factor 
        when $\eps$ is constant for all values of $m$.

\item A distributed uniformity tester for $\ell$ samples per machine with communication complexity $O(\sqrt{n\log(n)/\ell}/\eps^2)$. 
         (Note that when $\ell=1$, this beats the trivial algorithm by a $\sqrt{\log(n)}$ factor.) 
         We also give a matching lower bound when $\ell$ is not too big and $\eps$ is not too small. 
         The former of these constraints on the lower bound is necessary. Indeed, we 
         give a different algorithm with communication $O(n\log(n)/(\ell^2 \eps^4))$, which beats the previous
         bound when $\ell$ is sufficiently large. 

\item A one-pass streaming algorithm for closeness testing that uses $O(n\sqrt{\log(n)/m}/\eps^2)$ samples 
        for a wide range of values of $m$, and a distributed closeness tester with communication complexity
        $O(n^{2/3}\log^{1/3}(n)/(\ell^{2/3}\eps^{4/3}))$. 
        (Note that for $\ell=1$, this improves by a $\log^{2/3}(n)$ factor over the naive algorithm.)
\end{enumerate}
Our results are summarized in Tables~1 and~2.

\begin{table}[ht]\centering\small
   \scalebox{1}{ 
  \begin{tabular}{c|c|c|c|}
   {}& \multicolumn{3}{c|}{\bf Sample Complexity Bounds with Memory Constraints} \\\hline
   \textbf{Property} & \textbf{Upper Bound} & \textbf{Lower Bound 1} & \textbf{Lower Bound 2} \\\hline
   Uniformity & $O\left(\frac{n\log n}{m\eps^4}\right)$ & $\Omega\left(\frac{n\log n }{m\eps^4}\right)$ &$ \Omega\left(\frac{n}{m\eps^2}\right)$ \\\hline
   Conditions &$ n^{0.9} \gg m \gg \log(n)/\eps^2$ &$ m  = \tilde{\Omega} (\frac{n^{0.34}}{\eps^{8/3}}+\frac{n^{0.1}}{\eps^4})$ & Unconditional \\\hline\hline
   Closeness & $O(n\sqrt{\log(n)}/(\sqrt{m}\eps^2))$ & -- & --  \\\hline
   Conditions & $\tilde{\Theta}(\min(n, n^{2/3}/\eps^{4/3})) \gg m \gg \log(n)$ & -- & --  \\\hline
  \end{tabular}
  }
\caption{\label{fig:table:results} Summary of our Streaming Upper and Lower Bounds. 
}
\end{table}

\begin{table}[ht]\centering\small

   \scalebox{1}{ 
  \begin{tabular}{c|c|c|c|c|c|}
   {} & \multicolumn{5}{c|}{\bf Communication Complexity Bounds}\\\hline
   \textbf{Property}  & \textbf{UB 1} & \textbf{UB 2} & \textbf{LB 1} & \textbf{LB 2} & \textbf{LB 3}\\\hline
   Uniformity & $O\left(\frac{\sqrt{n\log(n)/\ell}}{\eps^2}\right)$ & $O\left( \frac{n\log(n)}{\ell^2\eps^4}\right)$ & $\Omega\left(\frac{\sqrt{n\log(n)/\ell}}{\eps^2} \right)$ & {$\Omega(\frac{\sqrt{n/\ell}}{\eps})$}&{$\Omega(\frac{n}{\ell^2\eps^2\log n})$}\\\hline
   Conditions & $\frac{\eps^8 n}{\log n} \gg \ell \gg \frac{\eps^{-4}}{n^{0.9}}$ & $\ell \ll \frac{\sqrt{n}}{\eps^2}$ & $\eps^{4/3}n^{0.3} \gg \ell $&{$\ell=\tilde{O}\left( \frac{n^{1/3}}{\eps^{4/3}}\right)$} &{$\ell=\tilde{\Omega}\left(\frac{n^{1/3}}{\eps^{4/3}}\right)$} \\\hline\hline
   Closeness  & $O\left( \frac{n^{2/3}\log^{1/3}(n)}{\ell^{2/3}\eps^{4/3}}\right)$ & - & - &-&-\\\hline
   Conditions & $n\eps^4/\log(n) \gg \ell$ & - & - &-&-\\\hline
  \end{tabular}
  }
\caption{\label{fig:table:results} Summary of our Distributed Upper and Lower Bounds. 
}
\end{table}

\subsection{Overview of our Approach}\label{ssec:approach}

In this section, we provide a detailed sketch of our algorithms and lower bounds.

\paragraph{Uniformity Testing Algorithms}
\new{We start by describing the main ideas underlying our efficient
uniformity testing protocols. We note that in both the distributed and the streaming settings,
our uniformity testers rely on a unified idea that we term {\em bipartite collision testing}.}
We remind the reader that testing uniformity based on the number of pairwise collisions
is the oldest algorithm in distribution testing~\cite{GR00}, which is now known to be
sample-optimal in the centralized setting~\cite{DiakonikolasGPP16}. \new{(In fact, it turns out that all known
efficient uniformity testers rely to some extent on counting pairwise collisions.)}
Recall that the collisions-based uniformity tester~\cite{GR00} takes $N$ samples
from the unknown distribution $p$ on $[n]$ and counts the total number of pairwise
collisions among them. If $p$ is the uniform distribution, the expected number of collisions will be
$(1/n) \cdot \binom{N}{2}$. On the other hand, if $p$ is $\eps$-far from the uniform distribution,
the expected number of collisions will be at least $(1/n) \cdot \binom{N}{2} (1+\Theta(\eps^2))$.
A now standard analysis --- involving carefully bounding the variance of this estimator
followed by an application of Chebyshev's inequality --- shows that for
$N = \Omega(\sqrt{n}/\eps^2)$ we can distinguish between the two cases
with high constant probability.

Unfortunately, the standard collisions-based tester described above involves
computing the total number of collisions among {\em all} pairs of samples,
and it is unclear if it can be implemented with non-trivial communication or memory.
\footnote{We note that the trivial protocol based on a total of $N$ samples
uses $N \cdot \log n$ bits of memory  and $N \cdot \log n$ bits of communication.}
To improve on these naive bounds, we propose a modified uniformity tester
(in the collocated/centralized setting) that we can implement
in the memory and communication restricted settings we study.
In particular, we consider a {\em bipartite collision tester} that works as follows:
We draw two independent sets of samples $S_1$ and $S_2$ from the unknown distribution $p$
and count the number of pairs of an element of $S_1$ and an element of $S_2$ that
collide. Importantly, we will use this scheme in such a way so that the first set of samples
$S_1$ will be {\em substantially smaller} than the second set of samples $S_2$.
Roughly speaking, our algorithm will store the set $S_1$ exactly, while
for each element of the set $S_2$, it will only need to know the number of collisions
with elements of $S_1$. This last step will allow us to save on space or communication.
An important technical condition that is required for our bipartite tester to succeed is
that  $|S_1| \cdot |S_2| \gg n/\eps^4$.

We now provide an overview of the analysis.  As is standard, we need to show that the number of
collisions is much larger in the non-uniform case than in the uniform case. To achieve that,
we consider for fixed $S_1$ the sum $p(S_1) = \sum_{s \in S_1} p_s$. We note that the expected number
of collisions is just $p(S_1) \cdot |S_2|$, and by standard concentration bounds one can show
that it will likely be close to this value. However, the average size of $p(S_1)$ is
$|S_1| \cdot \|p\|_2^2$, which is somewhat larger for non-uniform $p$
than for $p$ uniform. The detailed analysis is given in Section \ref{ssec:bipartite-alg}.

We note that our bipartite collision tester can be easily implemented in the memory
and communication bounded settings. In the former setting (Section~\ref{ssec:unif-alg-memory}),
it leads to a uniformity tester with sample complexity \new{$O(n \log n/(m\eps^4))$,
where $m$ is the bits of memory used.} In the distributed setting (Section~\ref{ssec:alg-distr-unif-small}),
when each machine stores $\ell$ samples, it leads to a tester with
sample complexity \new{$O(\sqrt{n\ell\log n}/\eps^2)$}
and communication cost \new{$O(\sqrt{(n \log n)/\ell}/\eps^2)$ bits}.
\new{(It should be noted that the above memory and communication upper bounds
match our lower bounds in some regimes of parameters.
We show that the tradeoff between sample complexity and memory/communication is inherent.
See Theorem~\ref{thm:memory_lb} and Corollary~\ref{cor:com-sample-tradeoff}.)}

It should be noted that when $\ell$ is sufficiently large, we can design a uniformity tester
whose communication beats the above (Section~\ref{ssec:alg-distr-unif-large}).
This fact should not be surprising since if $\ell \gg \sqrt{n}/\eps^2$, a single machine
could run a uniformity tester and simply return the answer. If $\ell$ is somewhat
smaller than this value, we can still take advantage of the large number of samples per machine.
The basic idea is for each player to communicate the number of pairwise
collisions {\em among their own samples}. As there are $O(\ell^2)$ pairs of samples
per machine, this will require roughly $n/\ell^2$ machines before we start
seeing any collisions, and this will give our approximate complexity
for large $\eps$, which can be seen to be better than our aforementioned bound,
when $\ell$ is sufficiently large.

\paragraph{Information-Theoretic Lower Bounds.}
We start by describing our memory lower bounds followed
by our communication lower bounds.

\smallskip

\noindent {\em Memory Lower Bound.}
We define a family of distributions on $[n] \times [2]$, i.e., supported on $2n$ elements
that will make uniformity testing difficult. In particular, we think of these distributions
as consisting of $n$ pairs of bins where the probability of landing in each pair is exactly $1/n$.
Equivalently, a distribution from this family can be thought of as
having $n$ (possibly biased) coins. The distribution picks a uniform
random coin, flips it, and returns the pair of the coin and the result.
If all coins are fair, we have the uniform distribution over $[n] \times [2]$.
On the other hand, if each coin is $\eps$-biased in a randomly chosen
direction, we have a distribution that is $\eps$-far from uniform. We
show that these two cases are hard to distinguish from each other
without expending a substantial amount of our computational resources.
Note that the standard sample complexity lower bounds for uniformity testing~\cite{Paninski:08}
rely on essentially the same hard instances.

We consider running a testing algorithm on a distribution that is randomly
either uniform or a random $\eps$-biased distribution as described above.
Let $X$ be the bit that describes whether or not the distribution is
uniform or not. We consider the shared information between $X$
and the memory of our algorithm after seeing $k$ samples.
We will attempt to show that this increases by at most $O(\eps^2 m/n)$ per step,
\new{where $m$ is the upper bound on the memory,}
which implies that it will take $\gg n/(m \eps^2)$ steps in order to reliably determine $X$.
The idea of the proof is fairly simple: The choice of which of the $n$ coins
is flipped is uniform no matter what, and so this choice does not tell
us anything about $X$. What might tell us something is the result of
flipping this coin, but two things make this difficult for us. 
On the one hand, the coin is at most $\eps$-biased, 
so knowing the result of the flip can only provide $O(\eps^2)$ information about this bias.
More critically, a biased coin is equally likely to be
biased in either direction, so knowing the result by itself still
tells us nothing unless we have some prior information about the direction of the possible bias. However, since our streaming
algorithm can store only $m$ bits of memory, on average it has at most
$O(m/n)$ bits of information about the bias of the (randomly chosen)
coin of the next flip. This argument can be formalized to show
that the next flip can only contribute $O(m/n)$ bits to the shared
information between $X$ and our memory. See Section~\ref{sec:mem-unif-lb}.

It may seem that the above argument is tight. 
With $m$ bits of information, one could know the biases (if they exist) of $m$ of the coins, 
and then each flip will (with probability $m/n$) give us $\eps^2$ information 
to accept or reject our hypothesis that these biases are real. 
However, we would need to be extraordinarily lucky to have this information. 
The only way that we could know the biases of these $m$ coins is if in 
our previous set of samples we had seen each of these coins $m$ times. 
This will prove difficult for two reasons: First, unless the number of samples is quite large, 
we would not expect to see \emph{any} coin show up in more than, say, $10$ samples. 
This will limit the amount of information we could expect to know about the bias of any given 
coin to $O(\eps^2)$. Additionally, we will have information about the biases of all 
of these particular $m$ coins only if all of them have shown up in our set of samples. 
But if our total number of samples is substantially sub-linear, this will only happen 
with probability $n^{-\Omega(m)}$. On the other hand, since there are only $2^m$ 
possible memory states, with high probability, we must be in one that occurs 
with probability only as small as $2^{-\Omega(m)}$, so we should only hope 
to have this information about $m/\log(n)$ coins (which we could get by, say, 
recording the coins and the results of the first $m/\log(n)$ samples). 
To formalize this intuition, we let $r$ be the vector whose $i^{th}$ entry 
is the expected number of times we have seen the $i^{th}$ coin based 
on the transcript, and note that the information gained is proportional to $\|r\|_2^2$. 
This can only be large if there is some unit vector $w$ with $w\cdot r$ large. 
But this would in turn mean that, conditioned on our (reasonably high probability) 
transcript, the average of $w_C$ over our samples $C$ is big. 
We show that this cannot happen with high probability by a Chernoff bound.
The details are deferred to Section~\ref{sec:mem-unif-lb}.

\smallskip

\noindent {\em Communication Complexity Lower Bound.}
To make this argument work for our communication complexity setting (Section~\ref{sec:com-unif-lb}),
we will need to restrict the model, in particular requiring that
the referee sees the players in sequence, never returning to one after
it has moved onto the next. In order to get our lower bounds for many
samples per player, we can proceed by immediate reduction to the
streaming model. Our algorithm stores the transcript of the
communication thus far. When the referee talks to a new player, we
record the samples that this player has, and then we compress this
down to the extended transcript involving the answers to the questions
asked of this new player. Note that at any time the amount of memory
used is at most $O(|T|+s \log(n))$, where $T$ is the final communication
transcript.

\paragraph{Closeness Testing Algorithms}
We now describe our algorithms for closeness testing (Section~\ref{sec:closeness-alg}).
In the communication model, a reasonable approach would be to follow the methodology of
~\cite{DK16}, by first using a few samples to flatten and then applying
some collision-based tester similar to that used for our uniformity testers.
We instead give an algorithm that turns out to be slightly more communication efficient.
The essential idea is to pick a random subset $W$ of the domain, and test whether
or not $p$ conditioned on $W$ is close to $q$ conditioned on $W$. If we take $|W|$
to be about $n/(s \log(n))$, we note that each machine can send either one bit (encoding
that they have no samples in $W$) or $\log(n)$ bits giving one such sample.
The former outcome happens about $\log(n)$ times more often, and so in total we need
to send about $\log(n)$ bits per element of $S$ received. However, for
constant $\eps$, we should need only about $|W|^{2/3}$ such samples to run a
standard closeness tester.

There is one substantial issue with this plan however. We need it to
be the case that $\Prob[W]$ is approximately $|W|/n$ and that
$\dtv(p|W, q|W) \approx \dtv(p,q)$. We note that both of these happen with high probability so
long as none of $p, q$ or $p-q$ are dominated by a small number of bins.
In particular, for $\eps$ large, it would be sufficient to know that
$\|p\|_{2}^2 , \|q\|_{2}^2 \ll 1/(s \log(n))$. Although this might not be the case for the given $p$ and $q$ we can obtain it by flattening. 
The details of the analysis are given in Section~\ref{ssec:close-com}.

For the streaming algorithm for closeness (Section~\ref{ssec:close-memory}),
we take a somewhat different approach. We select a random hash function
$h:[n] \to [m]$ and instead compare the distributions $h(p)$ and $h(q)$.
It is not hard to show that if $p=q$ then $h(p)=h(q)$, and that if $p$ and $q$ are far in
variation distance that $h(p)$ and $h(q)$ are likely somewhat far in
variation distance. The algorithm can record the counts of the
number of samples from each distribution in each bin in $O(m \log(n))$
samples and these counts are enough to run standard closeness testing
algorithms.

\subsection{Related work}

\paragraph{Distribution Testing.}

The field of \emph{distribution property testing}~\cite{BFR+:00} has been extensively investigated in the past
couple of decades, see, e.g., the surveys~\cite{Rub12, Canonne15,Gol:17}.
A large body of the literature has focused on characterizing the sample size needed to test properties
of arbitrary discrete distributions \new{in the centralized setting}.
This broad inference task originates from the field of statistics
and has been extensively studied in hypothesis testing~\cite{NeymanP, lehmann2005testing}
with somewhat different formalism.
The centralized regime is fairly well understood in a variety of settings:
for many properties of interest there exist sample-efficient testers
~\cite{Paninski:08, CDVV14, VV14, DKN:15, ADK15, CDGR16, DK16, DiakonikolasGPP16, CDS17, Gol:17, DGPP17, BatuC17, DKS17-gu, CDKS18}.
More recently, an emerging body of work has focused on leveraging \textit{a priori} structure
of the underlying distributions to obtain significantly improved sample
complexities~\cite{BKR:04, DDSVV13, DKN:15, DKN:15:FOCS, CDKS17, DaskalakisP17, DaskalakisDK16, DKN17}.

\paragraph{Distributed Statistical Inference.}
There has been substantial recent interest in distributed
estimation with communication constraints. A series
of works~\cite{zhang2013information, garg2014lower, BravermanGMNW16, JordanLY16, DGL+17, HanOW18, HanMOW18}
studied distributed {\em learning}
in both parametric and nonparametric settings obtaining (nearly-) matching upper and lower bounds.
Other learning tasks has been studied as well in the distributed setting, including regression~\cite{ZhuL18}
and PCA~\cite{kannan2013principal, liang2014improved}.

Classical work in information theory~\cite{Cover69, Ahlswede86} studies {\em simple} hypothesis testing
problems with communication constraints (as opposed to the composite testing problems studied in our work).
Their results and techniques appear to be orthogonal to ours. Two recent works~\cite{ACT18, ACT19}
give algorithms and lower bounds for distribution testing in a distributed model
where each machine holds a single sample and is allowed to communicate $\ll \log n$ bits of information.
We note that this communication model is very different from ours: First, our focus is on the
multiple samples per machine setting, where significant savings over the naive protocols is attainable.
Moreover, we do not impose any restrictions on the amount of information communicated by any individual machine,
and our goal is to minimize the total communication complexity of the protocol. Another recent work
~\cite{AMS18} studies distributed closeness testing in a {\em two party setting} where each party
has access to samples from one of the two distributions. This model is different than ours for two reasons:
First, we consider a large number of machines (parties). Second, we do not make the assumption
that each machine holds samples from different distributions. \new{Finally, recent work~\cite{Fischer18}
studied uniformity testing in the standard LOCAL and CONGEST models when each machine
holds a single sample. We note that their results and techniques seem incomparable to ours, 
and in particular have no implications on the communication complexity of uniformity testing in our broadcast model.
In the aforementioned models, there is an underlying graph whose nodes are the machines (players). 
In each round, every player can send information to all its neighbors. 
The objective is to minimize the number of rounds of the protocol, as opposed 
to the number of bits communicated. In particular, for the special case that the underlying
graph is a clique, their algorithm does not achieve non-trivial communication complexity.}

The work of~\cite{ChienLM10} studied distribution testing with limited memory
focusing on streaming algorithms in the framework of the canonical tester~\cite{PV11sicomp}. 
In particular, they show that for testing problems in this framework 
with sample complexity $f$, there are streaming algorithms with memory $m$ 
and sample complexity $f^2 \cdot 2^{O(\sqrt{\log(n)})}/m$. We note that this is competitive with 
our uniformity tester up to the $2^{O(\sqrt{\log(n)})}$ factor, but does substantially worse 
than our closeness tester for small $m$. They also consider another class of streaming algorithms, 
when the number of samples is large enough to learn the distribution in question, 
and when the problem of computing empirical distances has an efficient streaming algorithm. 
This does not compare favorably to our algorithms, which always use sub-learning sample complexity.

\subsection{Organization} The structure of this paper is as follows:
After some preliminaries in Section~\ref{sec:prelims}, 
in Section~\ref{sec:unif-alg} we present our bipartite collision tester and its 
implications to memory efficient and distributed uniformity testing.
Section~\ref{sec:unif-lb} presents our memory and communication lower bounds for uniformity testing.
In Section~\ref{sec:closeness-alg}, we give our upper bounds for closeness testing.
Finally, we conclude in Section~\ref{sec:conc} with a set of open problems for future work.

\section{Preliminaries} \label{sec:prelims}

In this section, we introduce the mathematical notation and background
necessary to state and prove our results in the following sections.

\smallskip

\noindent {\bf Notation}
We write $[n]$ to denote the set $\{1, \ldots, n\}$. 
We consider discrete distributions over $[n]$, which are functions
$p: [n] \rightarrow [0,1]$ such that $\sum_{i=1}^n p_i =1.$ 
We use the notation $p_i$ to denote the probability of element
$i$ in distribution $p$. The $\ell_1$ (resp. $\ell_2$) norm of a distribution 
is identified with the $\ell_1$ (resp. $\ell_2$) norm of the corresponding vector, i.e.,
$\|p\|_1 = \sum_{i=1}^n |p_i|$ and $\|p\|_2 = \sqrt{\sum_{i=1}^n p^2_i}$. 
The $\ell_1$ (resp. $\ell_2$) distance between (pseudo-)distributions 
$p$ and $q$ is defined as the  the $\ell_1$ (resp. $\ell_2$) norm of the vector of their difference, 
i.e., $\|p-q\|_1 = \sum_{i=1}^n |p_i -q_i|$ and
$\|p-q\|_2 = \sqrt{\sum_{i=1}^n (p_i-q_i)^2}$.   

For a random variable $X$, we denote by $\E[X]$ its expectation, $\Var[X]$ its variance, 
and $\Prob[\mathcal{E}]$ will denote the probability of event $\mathcal{E}$.
Sometimes we will use the notation $\E_{D}[\cdot], \Prob_D[\cdot]$ to make the underlying distribution explicit.

\smallskip

\noindent {\bf Distribution Testing} Distribution property testing
studies the following family of problems: 
given sample access to one or more unknown distributions,
determine whether they satisfy some global property or are ``far''
from satisfying the property.
In this work, we study the properties of identity testing (goodness of fit) 
and closeness testing (two-sample testing) between two distributions.

\begin{definition}[Identity Testing/Uniformity Testing]
The identity testing problem is the following: 
Given samples from an unknown distribution $p$ on $[n]$, a known distribution $q$ on $[n]$,
and a parameter $0<\eps<1$, we want to distinguish, with probability at least $2/3$, 
between the cases that $p = q$ versus $\|p-q\|_1 \geq \eps$. The special case corresponding to 
$q = U_n$, the uniform distribution on $[n]$, is the problem of uniformity testing.
\end{definition}

\begin{definition}[Closeness Testing]
The closeness testing problem is the following: 
Given samples from two unknown distributions $p, q$ on $[n]$,
and a parameter $0<\eps<1$, we want to distinguish, with probability at least $2/3$, 
between the cases that $p = q$ versus $\|p-q\|_1 \geq \eps$.
\end{definition}

\section{Communication and Memory Efficient Uniformity Testing} \label{sec:unif-alg}

In this section, we design our protocols for uniformity testing.
In Section~\ref{ssec:bipartite-alg}, we start with our bipartite uniformity tester.
In Section~\ref{ssec:unif-alg-memory}, we give our streaming uniformity tester, 
which follows as an application of our bipartite uniformity tester.
In Section~\ref{ssec:alg-distr-unif-small}, we give our first distributed uniformity tester,
which is an instantiation of our bipartite collision tester and performs well when 
the number of samples per machine $\ell$ is not very large.
For the complementary setting that $\ell$ is large, we give a different protocol 
in Section~\ref{ssec:alg-distr-unif-large}.

\subsection{Testing Uniformity via Bipartite Collisions} \label{ssec:bipartite-alg}

Our bipartite collision-based tester is described in pseudo-code below:

\medskip

\fbox{\parbox{5.8in}{

{\bf Algorithm} \textsc{Bipartite-Collision-Uniformity}$(p, n, \eps)$ \\
Input: Sample access to distribution $p$ over $[n]$ and $\eps>0$.\\
Output: ``YES'' if $p=U_n$; ``NO'' if $\|p-U_n\|_1 \ge \eps.$
\begin{enumerate}
	\item  Draw a multiset $S_1$ of $N_1$ i.i.d. samples from $p$.  
    \item For all $j \in [n]$ compute $a_j = | \{s \in S_1: s = j\}|$, i.e., the multiplicity of $j \in [n]$ in $S_1$. 
    \item \label{step:pre-process} If $\max_{j \in [n]} a_j >10$,  return ``NO''. Otherwise, continue with next step.
   \item Draw a multiset $S_2$ of $N_2$ i.i.d. samples from $p$. For $k \in [N_2]$, 
            let $b_k$ be the number of times that the $k$-th sample in $S_2$ appears in $S_1$.   
   \item \label{step:z} Let $Z=(1/N_2) \sum_{k=1}^{N_2} b_k$ and $T \eqdef \frac{N_1}{n}(1+\frac{\eps^2}{50})$. 
    \item If $Z \geq T$ return ``NO''; otherwise, return ``YES''.
\end{enumerate}
}}

\bigskip

The main result of this section is the following theorem:

\begin{theorem}\label{thm:bipartite-uniformity-alg}
Suppose that $N_1, N_2$ satisfy the following conditions: (i) $ \Omega(\eps^{-6}) = N_1  \leq n^{9/10}$
and $N_1 \cdot N_2  = \Omega(n/\eps^4)$, where the implied constants in the $\Omega(\cdot)$
are sufficiently large. Then the algorithm \textsc{Bipartite-Collision-Uniformity} distinguishes
between the cases that $p = U_n$ versus $\|p-U_n\|_1\ge \eps$ with probability at least $2/3$.
\end{theorem}

This section is devoted to the proof of Theorem~\ref{thm:bipartite-uniformity-alg}.
We start with the following definition:
\begin{definition}[Probability Mass of Multiset] \label{def:prob-multiset}
Let $p$ be a discrete distribution over $[n]$ and 
$S$ be a multiset of elements in $[n]$. We define \new{the probability mass 
of the multiset $S$} as follows: $p(S) \eqdef \sum_{j=1}^n a_j p_j$, where
$a_j$ is the number of occurrences of $j \in [n]$ in $S$.
\end{definition} 

It should be noted that we will use the above quantity for $S$ being a multiset 
of i.i.d. samples from the distribution $p$. In this case, $p(S)$ is a random variable
satisfying the following:

\begin{claim} \label{claim:exp-variance-prob-multiset}
Let $S$ be a multiset of $m$ i.i.d. samples from the distribution $p$ on $[n]$.
Then, we have that (i) $\E_{S}[p(S)] = m \cdot \|p\|_2^2$ and 
(ii) $\Var_{S}[p(S)] =  m \cdot (\|p\|_3^3 - \|p\|_2^4).$
\end{claim}
\begin{proof}
By definition, $p(S) \eqdef \sum_{j=1}^n a_j p_j$, where $a_j \sim \mathrm{Binomial}(m, p_j)$, 
$j \in [n]$. For the expected value, we can write:
$\E_{S}[p(S)] = \sum_{j=1}^n \E_{S}[a_j] p_j = \sum_{j=1}^n (m p_j) p_j  = 
m \cdot \sum_{j=1}^n p_j^2 = m \cdot \|p\|_2^2$.
To calculate the variance, we note that $p(S)$ can be equivalently expressed as 
$p(S) = \sum_{i=1}^{m} r_i$, where for $i \in [m]$ we have that 
$r_i = p_j$ with probability $p_j$, for $j\in[n]$. 
Note that the $r_i$'s are i.i.d. random variables and that 
$\E_{S}[r_i] = \|p\|_2^2$, $\E_{S}[r_i^2] = \|p\|_3^3$.
Therefore, we obtain that
$
\Var_{S}[p(S)] = \sum_{i=1}^{m}\Var_{S}[r_i] = m \cdot (\|p\|_3^3 - \|p\|_2^4) \;.$
This completes the proof of Claim~\ref{claim:exp-variance-prob-multiset}.
\end{proof}

We now proceed to establish the completeness and soundness of our tester. 

\paragraph{Completeness Analysis}
We will show that if $p=U_n$, then the tester outputs ``YES'' with probability at least $2/3$. 
We start by noting that it is very unlikely that the tester rejects in Step~\ref{step:pre-process}.

\begin{claim} \label{claim:no-larga-alpha}
Let $\mathcal{E} \eqdef \{ S_1: \max_{j \in [n]} a_j \leq 10\}$.
For $N_1 \leq n^{9/10}$, we have that 
$\Prob_{S_1} [\mathcal{E}] \geq 19/20$.
\end{claim} 
\begin{proof}
The probability that there exists some domain element in $[n]$ that appears at least $k$ times
in $S_1$, where $|S_1| = N_1$, is at most $n \cdot \binom{N_1}{k} n^{-k} \leq n \cdot \frac{N_1^k}{k! n^k}$.
By our assumption that $N_1 \leq n^{9/10}$, for $k=10$ the above probability is at most $1/k!$, 
which gives the claim.  
\end{proof}

We proceed to analyze the behavior of the random variable $Z$ defining our test statistic in Step~\ref{step:z}.
Note that $Z$ is the empirical estimate of $p(S_1)$ 
based on the multiset of samples $S_2$. Indeed, denoting by
$(\wh{p}_j)^{S_2}$ the empirical probability of $j \in [n]$ based on $S_2$, 
then we have that $Z = (1/\new{N_2}) \cdot \sum_{k=1}^{\new{N_2}} b_k = 
\sum_{j=1}^n a_j \cdot (\wh{p}_j)^{S_2}$.
We have the following simple claim:
\begin{claim} \label{claim:exp-var-z}
Let $p$ be any distribution over $[n]$. Then we have that:
(i) $\E_{S_2}[Z] = p(S_1)$  and (ii) $\Var_{S_2}[Z] = (1/\new{N_2}) \cdot \left(\sum_{j=1}^n a_j^2p_j - p^2(S_1)\right)$.
\end{claim}
\begin{proof}
Since $Z = \sum_{j=1}^n a_j \cdot (\wh{p}_j)^{S_2}$
and $\E_{S_2}[(\wh{p}_j)^{S_2}] = p_j$, we get that  
$\E_{S_2}[Z] = \sum_{j=1}^n a_j \cdot p_j = p(S_1)$.
To calculate the variance, we use the equivalent definition of 
$Z = (1/\new{N_2}) \cdot \sum_{k=1}^{\new{N_2}} b_k$, where 
for each $k \in [\new{N_2}]$ the i.i.d. random variables $b_k$ are defined as follows: 
$b_k=a_j$ with probability $p_j$, for $j \in [n]$. It follows that 
$\E_{S_2}[b_k] = \sum_{j=1}^n a_j p_j = p(S_1)$ and 
$\E_{S_2}[b^2_k] = \sum_{j=1}^n a_j^2 p_j$. Therefore, we get that
$\Var_{S_2}[Z] = (1/\new{N_2}^2) \cdot \new{N_2} \cdot \Var_{S_2}[b_1] = (1/\new{N_2}) \cdot (\sum_{j=1}^n a_j^2p_j - p^2(S_1))$.
This completes the proof.
\end{proof}

\noindent We note that Claim~\ref{claim:exp-var-z} will be useful both in the completeness
and the soundness cases.

To establish the correctness of the tester in the completeness case, it suffices to show that 
$\Prob_{S_1, S_2} [Z > T] \leq 1/3$. Since the event $\mathcal{E}$ occurs with high constant probability over $S_1$, 
by a union bound, it suffices to show that $\Prob_{S_1, S_2} [Z > T \mid \mathcal{E}] \leq 1/10$. 
An application of Chebyshev's inequality yields that
\begin{equation}\label{complete-cheb}
\Prob_{S_1, S_2} \left[ |Z - \E_{S_1, S_2}[Z \mid \mathcal{E}]| > \gamma \mid \mathcal{E} \right] \leq 
\Var_{S_1, S_2}[Z \mid \mathcal{E}] / \gamma^2 \;.
\end{equation}
Note that in the completeness case ($p = U_n$) we have that $p(S_1) = N_1/n$ 
deterministically over $S_1$. Therefore, by Claim~\ref{claim:exp-var-z} (i), 
we get that $\E_{S_1, S_2}[Z \mid \mathcal{E}] = \E_{S_2}[Z] = p(S_1) = N_1/n$. Similarly, 
we obtain that 
\begin{eqnarray*}
\Var_{S_1, S_2}[Z \mid \mathcal{E}]  
&=& \E_{S_1\mid \mathcal{E}}\left[\Var_{S_2}[Z|S_1]\right] + \Var_{S_1\mid \mathcal{E}}\left[\E_{S_2}[Z] \right] \\
&\leq& (1/\new{N_2}) \cdot \max_{S_1 \in \mathcal{E}}\sum_{j=1}^n a_j^2p_j +0 \\ 
&\leq& (10/\new{N_2}) \cdot p(S_1) +0 =  10 N_1/ (N_2 \cdot n) \;,
\end{eqnarray*}
where the second line uses Claim~\ref{claim:exp-var-z} (ii) and 
the fact that $\Var_{S_1\mid \mathcal{E}}\left[\E_{S_2}[Z] \right]=0$,
and the third line follows from the definition of $\mathcal{E}$.
By setting $\gamma \eqdef (\eps^2/100) \E_{S_1, S_2}[Z \mid \mathcal{E}] = (\eps^2/100) p(S_1)$,
the RHS of (\ref{complete-cheb}) is at most $ O\left( n/( \eps^4 \cdot N_1 \cdot N_2) \right)$.
Recalling our assumption that $N_1 \cdot N_2 = \Omega(n/\eps^4)$ for a sufficiently 
large constant in the $\Omega()$, we get that with probability  
at least $9/10$ we have that $Z < (1+\eps^2/100) (N_1/n)  < T$.
This proves the completeness of our tester.

\paragraph{Soundness Analysis}
We will show that if $\|p-U_n\|_1 \ge \eps$ for $\eps$ satisfying 
$N_1 = \Omega(1/\eps^6)$, where the  constant in $\Omega()$
is sufficiently large, then the tester outputs ``NO'' with probability at least $2/3$.
The technical part of the soundness proof involves showing that 
$\E_{S_2}[Z] = p(S_1)$ will be significantly larger than its value of $m/n$ in the completenesss case.
Specifically, we show the following:

\begin{lemma}\label{lem:exp-gap}
Let $\mathcal{F} \eqdef \{ S_1: p(S_1) \geq (1+\eps^2/\new{40}) \cdot (N_1/n)\}$.
If $\|p-U_n\|_1 \ge \eps$ and $N_1 = \Omega(1/\eps^6)$, then
$\Prob_{S_1} \left[ \mathcal{F} \right] \geq 9/10$.
\end{lemma}
\begin{proof}
The proof proceeds by case analysis. First, we consider the case that
the distribution $p$ assigns sufficient probability on heavy elements. 
This case turns out to be fairly easy to handle. The complementary case
that $p$ has a small amount of mass on heavy elements requires
a more elaborate argument. \new{Let $\theta \eqdef 10^4/(\eps^2 n)$
and let $H = H(p) \eqdef \{i \in [n] \mid p_i \geq \theta \}$}
be the set of heavy bins. We consider the following cases:

\medskip

\noindent {\bf [Case I:  $p(H) = \sum_{i \in H} p_i \geq \eps^2/\new{1000}$.]} Let \new{$W = |S_1 \cap H|$.} 
Note that $\E_{S_1}[W]  = N_1 \cdot p(H) \geq N_1 \eps^2/\new{1000}$.
By a multiplicative Chernoff bound, for $\delta = 1/3$, we get that 
\[\Prob_{S_1}\left[ W \leq (1-\delta) \cdot N_1 \eps^2/\new{1000} \right] \leq e^{-N_1 \eps^2 \delta^2/200} < 1/10 \]
\new{where we used the fact that $N_1 = \omega(1/\eps^2)$}. That is, with probability at least $9/10$, 
the multiset $S_1$ will contain at least $N_1 \eps^2/150$ samples from the set $H$.
If this happens, then we have that 
\[p(S_1) > (N_1 \eps^2/\new{1500}) \cdot  \theta > (6N_1)/n >  (1+\eps^2/\new{40}) \cdot (N_1/n) \]
which proves Case I.

\medskip

\noindent {\bf [Case II: $p(H) = \sum_{i \in H} p_i \leq \eps^2/100$.]} Let 
$\overline{H} \eqdef [n]\setminus H$ and $S'_1 = S_1 \cap \overline{H}$, 
i.e., $S'_1$ contains the samples in $S_1$ that correspond to light elements of $p$. 
Clearly, we have that $p(S_1) \geq p(S'_1)$.
We will show that $\Pr_{S_1} [p(S'_1) \geq  (1+\eps^2/\new{40}) \cdot (N_1/n)] \geq 9/10.$

\new{
Since $p(H)\leq \eps^2/1000$, with probability at least $1-\eps^2/1000$,
a given sample in $S_1$ will also be in $S'_1$.  So, if $N'_1 = |S'_1|$ 
we have that $\E_{S_1}[N'_1] \geq N_1 (1-\eps^2/1000)$. By Markov's inequality
applied to $N_1-N'_1$, with probability at least $19/20$ over $S_1$, we have that 
$N'_1 \geq N_1 (1-\eps^2/50)$. We will henceforth condition on this event.}

To bound $p(S'_1)$ from below, we consider the normalized distribution, 
$p'$, over $[n]$ obtained from $p$ after removing its heavy elements. That is,
for $i \in \overline{H}$, $p'_i = p_i / (1-p(H))$; and for $i \in H$, $p'_i = 0$.
Note that $S'_1$ consists of $N'_1$ elements drawn i.i.d. from $p'$.
By definition, we have that $p(S'_1) = (1-p(H)) p'(S'_1)$.
We will bound $p'(S'_1)$ by applying Chebyshev's inequality.

We start by noting that $p'$ is also far from being uniform:
\begin{claim}\label{claim:p-prime-far}
We have that $\|p' - U_n\|_1\geq \eps/3$.
\end{claim}
\begin{proof}
Since $\|p - U_n\|_1 \geq \eps$, we have that 
$\sum_{j \in [n]: p_j > 1/n} | p_j - 1/n| > \eps/2$. 
Since $p(H) < \eps^2/100$, it follows that 
$\sum_{j: 1/n < p_i <  \theta} |p_j - 1/n|  > \eps/2 - \eps^2/100 > \eps/3$. 
Noting that if $1/n < p_j < \theta$, then $|p_j - 1/n| <|p'_j -1/n|$ gives the claim.  
\end{proof}    
By Chebyshev's inequality, we can write:
\begin{equation} \label{eqn:interm-cheb}
\Prob_{S'_1} \left[ |p'(S'_1)- \E_{S'_1}[p'(S'_1)]| > \gamma \right]  \leq  \Var_{S'_1}[p'(S'_1)]/ \gamma^2 \;,
\end{equation}
\new{where we will set $\gamma \eqdef (\eps^2/20) \cdot \E_{S'_1}[p'(S'_1)]$.}
By Claim~\ref{claim:exp-variance-prob-multiset}, we have that 
$\E_{S'_1}[p'(S'_1)] = N'_1 \cdot \|p'\|_2^2$,  and $\Var_{S'_1}[p'(S'_1)] =  N'_1 \cdot (\|p'\|_3^3 - \|p'\|_2^4)$.
In the following claim, we bound the variance from above:
\begin{claim} \label{claim:var-p-prime}
We have that $ \Var_{S'_1}[p'(S'_1)] < \frac{2000 N'_1 \cdot \|p'\|_2^2}{\eps^2 n}$.
\end{claim} 
\begin{proof} 
We have that $\Var_{S'_1}[p'(S'_1)]  < N'_1 \cdot \|p'\|_3^3 \leq N'_1 \cdot \|p'\|_{\infty} \cdot \|p'\|_2^2$. 
We will show that $\|p'\|_{\infty} \leq 2 \theta$ from which the claim follows. 
Indeed, note that the non-zero probability values of $p'$ are $p_i' = p_i / (1-p(H))$, for $i \notin H$. 
Since $p_i \leq  \theta$ and $p(H) \leq \eps^2/\new{1000} < 1/2$, we get the desired
bound on $\|p'\|_{\infty}$ and the claim follows.
\end{proof}
The RHS of (\ref{eqn:interm-cheb}) can be bounded by
$ O(\Var_{S'_1}[p'(S'_1)] / (\eps^4\E_{S'_1}[p'(S'_1)]^2)) = O(1/(N'_1 \cdot \eps^6))$,
where we used our bounds on the first two moments and the fact that $\|p'\|_2^2 \ge 1/n$. 
\new{Using the condition that $N_1 = \Omega (1/\eps^6)$ 
and the fact that $N'_1 \geq N_1 \cdot (1-O(\eps^2))$, 
the above probability is at most $1/20$.}

We now show that if $|p'(S'_1)- \E_{S'_1}[p'(S'_1)]| \leq \gamma$ the event $\mathcal{F}$ holds. 
By the definition of $\gamma$ and the fact $\E_{S'_1}[p'(S'_1)] =  N'_1 \cdot \|p'\|_2^2$, 
this is equivalent to $p'(S'_1) > (1-\eps^2/20) \cdot N'_1 \cdot \|p'\|_2^2$. 
Claim \ref{claim:p-prime-far} implies that $\|p'\|_2^2 \geq (1/n) \cdot (1+\eps^2/9)$, 
and therefore we get that
$$p'(S'_1) > (1-\eps^2/20)\cdot \new{(1-\eps^2/50) \cdot (N_1/n)} \cdot (1+\eps^2/9) > (N_1/n) \cdot (1+\eps^2/\new{35}) \;,$$
\new{where we used our lower bound on $N'_1$.} 
Finally, we have that 
$$ p(S_1) \geq p(S'_1) = (1-p(H)) p'(S'_1) \geq (1-\eps^2/\new{1000}) \cdot (N_1/n) \cdot (1+\eps^2/\new{35})
> (N_1/n) \cdot (1+\eps^2/\new{40}) \;,
$$
and the proof of Lemma~\ref{lem:exp-gap} is complete.
\end{proof}


To establish correctness in the soundness case, it suffices to show that 
$\Prob_{S_1, S_2} [Z \leq T] \leq 1/3$. To show this, we condition on any $S_1$ 
such that the events $\mathcal{E}$ and $\mathcal{F}$ hold. 
Note that if $\mathcal{E}$ does not occur, then our tester correctly rejects.
By Lemma~\ref{lem:exp-gap} above, $\mathcal{F}$ holds with probability at least $9/10$
over $S_1$. Hence, by a union bound, it suffices to show that 
$\Prob_{S_1, S_2} [Z \leq T \mid \mathcal{E}, \mathcal{F}] \leq 1/10$. 

An application of Chebyshev's inequality for 
$\gamma = (\eps^2/100)\E_{S_1, S_2}[Z \mid \mathcal{E}, \mathcal{F}]$ 
yields that
\begin{equation} \label{eqn:sound-cheb}
\Prob_{S_1, S_2} \left[ |Z - \E_{S_2}[Z]| > \gamma \mid \mathcal{E}, \mathcal{F} \right] 
\leq \Var_{S_1, S_2}[Z \mid \mathcal{E}, \mathcal{F}] / \gamma^2 \;.
\end{equation}
By Claim~\ref{claim:exp-var-z}, we have that $\E_{S_2}[Z] = p(S_1)$
and $\Var_{S_2}[Z] = (1/\new{N_2}) \cdot \left(\sum_{j=1}^n a_j^2p_j - p^2(S_1)\right)$.
\new{The same argument as in the completeness case gives} 
that $\Var_{S_1, S_2}[Z \mid \mathcal{E}, \mathcal{F}] < (10/\new{N_2}) \cdot p(S_1).$
By our choice of $\gamma$ and our bound on the variance the right-hand side of (\ref{eqn:sound-cheb}) 
is at most $ O\left(1/ (N_2 \eps^4 p(S_1))\right) $.
Recalling our assumption that $N_1 \cdot N_2 = \Omega(n / \eps^4)$ and the fact that
$p(S_1) \geq N_1/n$ (by Lemma~\ref{lem:exp-gap}, since $\mathcal{F}$ occurs), 
it follows that this probability is at most $1/10$.
By Lemma~\ref{lem:exp-gap} we have that $p(S_1) > (N_1/n) \cdot (1+\eps^2/\new{40})$.
Therefore, with probability at least \new{$8/10$ (by a union bound on the two error events)}, we have that $Z > (N_1/n) \cdot (1+\eps^2/50)$.
This establishes the soundness case and completes the proof of Theorem~\ref{thm:bipartite-uniformity-alg}.
 
\subsection{Memory Efficient Uniformity Testing} \label{ssec:unif-alg-memory}

In this section, we show how our bipartite collision tester can be used
to obtain a memory efficient single pass streaming algorithm.
Our memory efficient tester is described in pseudo-code below:

\medskip

\fbox{\parbox{5.8in}{

{\bf Algorithm} \textsc{Streaming-Uniformity}$(p, n, m, \eps)$ \\
Input: Sample access to distribution $p$ over $[n]$, memory bound $m$, and $\eps>0$.\\
Output: ``YES'' if $p=U_n$; ``NO'' if $\|p-U_n\|_1 \ge \eps.$
\begin{enumerate}
	\item  Draw a multiset $S_1$ of $N_1 \eqdef m / (2 \log n)$ i.i.d. samples from $p$. \new{Store $S_1$ in memory.}
    \item For all $j \in [n]$ compute $a_j = | \{s \in S_1: s = j\}|$, i.e., the multiplicity of $j \in [n]$ in $S_1$. 
    \item \label{step:pre-process} If $\max_{j \in [n]} a_j >10$,  return ``NO''. Otherwise, continue with next step.
   \item Draw a multiset $S_2$ of $N_2 \eqdef \Theta \left( n \log n / (m \eps^4) \right)$ i.i.d. samples from $p$, 
           for an appropriately large constant in $\Theta(\cdot)$. 
           For $k \in [N_2]$, let $b_k$ be the number of times that the $k$-th sample in $S_2$ appears in $S_1$.  
            \new{Store the partial sum $\sum_k b_k$ in a single pass.}
   \item \label{step:z2} Let $Z=(1/N_2) \sum_{k=1}^{N_2} b_k$ and $T \eqdef \frac{N_1}{n}(1+\frac{\eps^2}{50})$. 
    \item If $Z \geq T$ return ``NO''; otherwise, return ``YES''.
\end{enumerate}
}}

\bigskip

The following theorem is essentially a corollary of Theorem~\ref{thm:bipartite-uniformity-alg}
and characterizes the performance of the above algorithm:

\begin{theorem}\label{thm:streaming-uniformity-alg}
\new{Suppose that $m \geq \Omega(\log n /\eps^6)$ and $m \leq \tilde{O}(n^{9/10})$.}
Algorithm \textsc{Streaming-Uniformity} is a single pass streaming algorithm using at most
$m$ bits of memory, and distinguishes between the cases that $p = U_n$ versus $\|p-U_n\|_1\ge \eps$ 
with probability at least $2/3$.
\end{theorem}
\begin{proof}
The correctness of \textsc{Streaming-Uniformity} as a uniformity testing algorithm follows
from Theorem~\ref{thm:bipartite-uniformity-alg} and our choice of parameters. 
It is straightforward to check that the assumptions of the latter theorem are satisfied
for our choice of $N_1$ and $N_2$.

It is also easy to argue that the algorithm is implementable 
in the single pass streaming model with at most $m$ bits of memory. 
Step~1 of the algorithm uses $N_1 \log n \leq m/2$ bits of memory. 
We claim that Step~4 can be implemented in a single pass with at most $\log N_2+4$ bits of memory. 
Indeed, since no element of $[n]$ appears in $S_1$ more than $10$ times (since the algorithm did not reject in Step~3), 
each $b_k$ is at most $10$. \new{Therefore, $\sum_{k=1}^{N_2} b_k \leq 10 N_2$, 
and thus the sum $\sum_{k=1}^{N_2} b_k$ can be stored with $\log N_2 + 4$ bits of memory. 
In summary, the total memory used by our algorithm 
is at most $m/2+ \log N_2+4$. By our assumption that 
$m \geq \Omega(\log n /\eps^6)$, it follows that $N_2 \leq \eps^2 n$ or $\log N_2 < \log n$.
Therefore, $\log N_2 \ll m$ which completes the proof of Theorem~\ref{thm:streaming-uniformity-alg}.}
\end{proof}

\subsection{Distributed Uniformity Testing for Small Number of Samples per Machine} \label{ssec:alg-distr-unif-small}

Let $\ell$ denote the total number of samples from $p$ possessed
by each machine. When $\ell  = \tilde{O}(n^{1/3}/\eps^{4/3})$,
we use the following tester, which can be viewed as an instantiation of our bipartite collision tester.

\medskip

\fbox{\parbox{5.8in}{
{\bf Algorithm I} \textsc{Distributed-Bipartite-Uniformity}$(p, n, \ell, \eps)$ \\
Input: Unbounded number of machines, each with $\ell$ i.i.d. samples from $p$ and $\eps>0$.\\
Output: ``YES'' if $p=U_n$; ``NO'' if $\|p-U_n\|_1 \ge \eps.$
\begin{enumerate}
	\item The referee asks $m_1 = \Theta\left((1/(\eps^2\ell^{3/2}))\sqrt{n/\log n}\right)$ machines 
	         to reveal all their samples. Let $S_1$ be the resulting multiset of samples from $p$. 

    \item For all $j \in [n]$ the referee computes $a_j = | \{s \in S_1: s = j\}|$, i.e., the multiplicity of $j \in [n]$ in $S_1$. 
    \item \label{step:pre-process} If $\max_{j \in [n]} a_j >10$,  the referee returns ``NO''. Otherwise, we continue with next step.

    \item The referee queries a new set of $m_2=\Theta(\frac{n}{\eps^4 \ell^2 m_1})$ machines, indexed by $k \in [m_2]$,
             in increasing order of $k$, to report the value $b_k = \sum_{i=1}^{\ell} a_{s^k_i}$ 
             corresponding to their sample set $S^k_2 = \{s^k_i\}_{i=1}^{\ell}$.  Note that $b_k$ is the number of collisions of 
             $S^k_2$ with $S_1$.
             
    \item For $t \in [m_2]$, define $B_t = \sum_{k=1}^t b_k$. Let $Z= B_{m_2} / (\new{\ell \cdot m_2})$ 
             and $T \eqdef \frac{m_1 \cdot \ell}{n}(1+\eps^2/50)$.

    \item The referee computes $B_t$ in increasing order of $t \in [m_2]$. 
             If for some $t \in [m_2]$, $B_t \geq \new{\ell \cdot m_2} \cdot T$, we terminate and returns ``NO''.
             Otherwise, we have that $Z < T$ and we return ``YES''.
\end{enumerate}
}}

\medskip

The following theorem characterizes the performance of the above algorithm:

\begin{theorem}\label{thm:com-uniformity-alg}
\new{Suppose that $\Omega(1/\eps^6) \leq m_1 \cdot \ell \leq O(n^{9/10})$.}
Algorithm \textsc{Distributed-Bipartite-Uniformity} draws a total of $O\left((1/\eps^2) \sqrt{n \cdot \ell  \cdot \log n}\right)$
samples from $p$, uses at most $O\left((1/\eps^2) \sqrt{(n/\ell) \cdot \log n}\right)$ bits of communication, 
and distinguishes between the cases that $p = U_n$ versus $\|p-U_n\|_1\ge \eps$ 
with probability at least $2/3$.
\end{theorem}

\begin{proof}
 The correctness of \textsc{Distributed-Bipartite-Uniformity} follows
from Theorem~\ref{thm:bipartite-uniformity-alg} and our choice of parameters. 
It is straightforward to check that the assumptions of the latter theorem are satisfied
for our choice of $N_1 = m_1 \cdot \ell$ and $N_2 = m_2 \cdot \ell$.

It is also clear that the sample complexity of our algorithm is 
$$(m_1+m_2) \cdot \ell = O\left( (1/\eps^2) \sqrt{n/(\ell^3 \log n)} + 
(1/\eps^2) \sqrt{(n \log n) / \ell} \right) \cdot \ell = O\left( (1/\eps^2) \sqrt{n \cdot \ell \cdot \log n} \right) \;.$$
We now proceed to bound the communication complexity. Note that Step~1 uses 
$m_1 \cdot \ell \cdot \log n$ bits of communication. \new{We claim that Step~6 can be implemented with 
$O(m_2)$ bits of communication. To achieve this, we transmit each $b_k$ in unary by sending $b_k$ many $1$'s
followed by a $0$ and terminate the algorithm rejecting if the partial sum $B_t = \sum_{k=1}^t b_k$, for some $t \in [m_2]$, 
exceeds the rescaled threshold  $\ell \cdot m_2 \cdot T$. Since we use $b_k+1$ bits to encode each $b_k$, 
the number of bits of communication is bounded by $\max_t (B_t + t)$, 
which is at most $\ell \cdot m_2 \cdot T + m_2$. We now show that 
$\ell \cdot m_2 \cdot T = O(m_2)$, which proves the desired communication upper bound. 
Note that by our choice of $m_1, m_2$, we have that $\ell \cdot m_2 \cdot T = O(1/\eps^4)$
and moreover $m_2 > m_1 \cdot \ell  = \Omega(1/\eps^6)$. Therefore, $\ell \cdot m_2 \cdot T \leq m_2$, as desired.}
This completes the proof of Theorem~\ref{thm:com-uniformity-alg}.
\end{proof}

\subsection{Distributed Uniformity Testing for Large Number of Samples per Machine} \label{ssec:alg-distr-unif-large}
In this section, we give our alternate distributed uniformity tester
with improved communication complexity when the number
of samples per machine is large.

When the number of samples per machine $\ell$ satisfies
$\ell  = \tilde{\Omega}(n^{1/3}/\eps^{4/3})$, we will use the following algorithm:

\medskip

\fbox{\parbox{5.8in}{
{\bf Algorithm II} \textsc{Distributed-Aggregate-Uniformity}$(p, n, \ell, \eps)$ \\
Input: Each machine has $\ell$ samples from a distribution $p$ over $[n]$ and $\eps>0$.\\
Output: ``YES'' if $p=U_n$; ``NO'' if $\|p-U_n\|_1 \ge \eps.$
\begin{enumerate}
	\item The referee asks the first $\new{m} = \frac{\new{12800}\cdot n}{\ell^{2}\epsilon^4}$ machines 
	          to reveal the number of collisions \new{$c_k$, $k \in [m]$,} each of them see in their $\ell$ samples. 
    \item Compute the statistic \new{$Z = (1/m) \sum_{k=1}^m c_k$} and the Threshold 
             $T = {\ell\choose 2}\frac{1+\epsilon^2/2}{n}$.
    \item If $Z \ge T$ return ``NO''; otherwise, return ``YES''.
\end{enumerate}
}}

\medskip

The following theorem characterizes the performance of the above algorithm:

\new{
\begin{theorem}\label{thm:unif-many-sample}
The algorithm \textsc{Distributed-Aggregate-Uniformity} draws a total of $O\left(n/(\ell\eps^4)\right)$
samples from $p$, uses $O(\frac{n\log n}{\ell^2\eps^4})$ bits of communication, 
and distinguishes between the cases that $p = U_n$ versus $\|p-U_n\|_1\ge \eps$ 
with probability at least $2/3$.
\end{theorem}
}

In this rest of this section, we prove Theorem~\ref{thm:unif-many-sample}.

First note that the sample complexity of our algorithm is $m \cdot \ell=O(\frac{n}{\ell\eps^4})$.
Moreover, the communication complexity of our tester is $O(\frac{n}{\ell^2\eps^4}\log(\ell^2))=O(\frac{n\log n}{\ell^2\eps^4})$, 
since each of the $\Theta(\frac{n}{\ell^2\eps^4})$ players sends the number of their internal collisions, 
which can be at most ${\ell \choose 2}$. 

\medskip

We will compute the mean and variance of the statistic $Z$ 
and show that the uniform and non-uniform cases are well-separated. 
Lemma~2.3 in \cite{DiakonikolasGPP16} implies that:
$$
\E[Z]= \frac{1}{m} \sum_{i=1}^{m} \E[c_i]= \E[c_i]={\ell\choose 2}\|p\|_2^2 
$$ 
and
\begin{align*}
\Var[Z] =\new{\frac{1}{m^2}} \sum\limits_{i=1}^t\Var[c_i] 
&= \frac{\ell^{2}\eps^4}{\new{12800}n}\Big{[}{\ell\choose 2}(\|p\|_2^2-\|p\|_2^4) + \ell(\ell-1)(\ell-2)(\|p\|_3^3-\|p\|_2^4)\Big{]}\\
&\leq \frac{\ell^{2}\eps^4}{\new{12800}n} \left[ \ell^2 \|p\|_2^2+\ell^3(\|p\|_3^3-\|p\|_2^4)\right] \;.
\end{align*}

The following lemma, which is an adaptation of an analogous lemma in~\cite{DiakonikolasGPP16}), 
will give us the minimum number 
of samples per player required for the tester to work given our choice of parameters:

\begin{lemma}\label{lem:samples}
Let $\alpha$ satisfy $\|p\|_2^2 = \frac{1+\alpha}{n}$ and $\sigma$ be the standard deviation of $Z$. 
The number of samples required by \textsc{Distributed-Aggregate-Uniformity} is at most
\[
\ell \leq \sqrt{\frac{5 \sigma n}{|\alpha-\epsilon^2/2|}} \;,
\]
in order to get error probability at most $1/4$.
\end{lemma}

\begin{proof}
By Chebyshev's inequality, we have that
\[
\Prob\left[\;\left|Z-\E[Z]\right| \geq k \sigma\right]=\Prob\left[\;\left|Z-\binom{\ell}{2}\|p\|_2^2 \right| \geq k \sigma\right] \leq \frac{1}{k^2} \;,
\]
where $\sigma \triangleq \sqrt{\Var[Z]}$.

We want $Z$ to be closer to its expected value than the threshold is to that expected 
value because when this occurs, the tester outputs the right answer. 
Furthermore, to achieve our desired probability of error of at most $1/4$, 
we want this to happen with probability at least $3/4$. 
So, we want
\[
k \sigma \leq |\E[Z]-T| =\left| \binom{\ell}{2} \left(\|p\|_2^2 - \frac{1+\epsilon^2/2}{n}\right)\right| = \binom{\ell}{2} |\alpha - \epsilon^2/2| / n \;.
\]
For $\ell$ larger than some small constant and $k=2$, the following slightly stronger condition for $\ell$ suffices: 
\[
\sigma \leq \ell^2 \cdot \frac{|\alpha-\epsilon^2/2|}{5n}.
\]
So, it suffices to have
\begin{equation}\label{sufficient_cond}
\ell \geq \sqrt{\frac{5 \sigma n}{|\alpha-\epsilon^2/2|}}.
\end{equation}
We might as well take the smallest number of samples per player $\ell$ for which 
the tester works, which implies the desired inequality.
\end{proof}

To complete the proof, we need to show that 
given enough samples there is a clear separation between the completeness 
and soundness cases regarding the value of our statistic. 

By Lemma \ref{lem:samples}, it suffices to bound from above the variance $\sigma^2$. 
We proceed by case analysis based on 
whether the term $\ell^2 \|p\|_2^2$ or $\ell^3 (\|p\|_3^3 - \|p\|_2^4)$ 
contributes more to the variance.

\paragraph{Case when $\ell^2\|p\|_2^2$ is Larger}

%

We have the following lemma:

\begin{lemma}\label{lem:case1}
Let $\|p\|_2^2=(1+\alpha)/n$. Consider the completeness case when $\alpha=0$ 
and the soundness case when $\alpha \geq \epsilon^2$. 
If $\ell^2\|p\|_2^2$ contributes more to the variance, 
i.e., if $$\ell^2\|p\|_2^2 \geq \ell^3 (\|p\|_3^3 - \|p\|_2^4) \;,$$ 
then \new{\textsc{Distributed-Aggregate-Uniformity}} has error probability at most $1/4$ for any value of $\ell$.
\end{lemma}

\begin{proof}
Suppose that $\ell^2\|p\|_2^2 \geq \ell^3 (\|p\|_3^3 - \|p\|_2^4)$. 
Then $\sigma^2 \leq 2 \frac{\ell^4\eps^4}{12800n} \|p\|_2^2 =  \frac{\ell^4\eps^4}{6400n} (1+\alpha)/n$. 
Substituting this into \eqref{sufficient_cond} gives:
\[
\sqrt{\frac{5 \sigma n}{|\alpha-\epsilon^2/2|}}\leq \frac{ \ell\eps (1+\alpha)^{1/4}}{4\sqrt{|\alpha-\epsilon^2/2|}} \;.
\]
One can show the latter inequality, using calculus to maximize the ratio: 
$\frac{ \ell\eps (1+\alpha)^{1/4}}{4\sqrt{|\alpha-\epsilon^2/2|}}$ by varying $\alpha$. 
One gets that $\alpha=\epsilon^2$ maximizes the expression 
for $\alpha \in \{0\}\cup[\epsilon^2,n-1]$, since it is decreasing in the interval $[\epsilon^2,n-1]$ 
and also $\alpha=\eps^2$ gives a slightly larger value than $\alpha=0$. Thus, we get that: 
\[
\sqrt{\frac{5 \sigma n}{|\alpha-\epsilon^2/2|}}\leq \ell \cdot\frac{ \eps (1+\eps^2)^{1/4}}{4\sqrt{\epsilon^2/2}}\leq \ell\frac{(1+\eps^2)^{1/4}}{2}\leq \ell
\]
for any $\eps<1$. Therefore, this completes the proof since the requirements of Lemma~\ref{lem:samples} are satisfied.
\end{proof}

\paragraph{Case when $\ell^3(\|p\|_3^3 - \|p\|_2^4)$ is Larger}
In this case, we show the following:
\begin{lemma}
\label{lem:case2}
Let $\|p\|_2^2=(1+\alpha)/n$. Consider the completeness case when $\alpha =0$ 
and the soundness case when $\alpha \geq \epsilon^2$.
If $\ell^3 (\|p\|_3^3 - \|p\|_2^4)$ contributes more to the variance, 
i.e., if $$\ell^3 (\|p\|_3^3 - \|p\|_2^4) \geq \ell^2\|p\|_2^2 \;,$$
then for any $\ell\leq \frac{16\sqrt{n}}{3\eps^2}$ in the completeness case and any  $\ell\leq \frac{16\sqrt{n}}{3\alpha}$ in the soundness case, our tester \new{\textsc{Distributed-Aggregate-Uniformity}} achieves error probability at most $1/4$.
\end{lemma}

\begin{proof}
Suppose that $\ell^3 (\|p\|_3^3 - \|p\|_2^4) \geq \ell^2\|p\|_2^2$. 
Then $\sigma^2 \leq 2 \frac{\ell^5\eps^4}{12800n} (\|p\|_3^3 - \|p\|_2^4)$.
Substituting this into \eqref{sufficient_cond} gives:

\[
\ell\geq \sqrt{\frac{5 \sigma n}{|\alpha-\epsilon^2/2|}}\geq \frac{\ell^{5/4}\eps n^{1/4}(\|p\|_3^3 - \|p\|_2^4)^{1/4}}{4\sqrt{|\alpha-\epsilon^2/2|}}\Leftrightarrow
\]
\[
\ell\leq \frac{256 |\alpha-\epsilon^2/2|^2}{\eps^4 n (\|p\|_3^3 - \|p\|_2^4)}
\]

Let us parameterize $p$ as $p_i = 1/n + a_i$ for some vector $a=(a_1,\dots,a_n)$. 
Then we have $\|a\|_2^2 = \alpha/n$.

{
\noindent In the completeness case, \new{the above sufficient condition always holds since the right hand side is infinite 
(i.e., $\|p\|_3^3=\|p\|_2^4$).}
In the soundness case, we get the following sufficient condition:
\[
\ell \leq \frac{256 (\alpha/\new{2})^2}{\eps^4 n (\|p\|_3^3 - \|p\|_2^4)}  \quad\text{(since $\epsilon^2 \leq \alpha$)} \;.
\]
We also have that:
\begin{align*}
\|p\|_3^3 - \|p\|_2^4 &\leq \|p\|_3^3 - \frac{1}{n^2} =\left[\sum_{i=1}^n (1/n + a_i)^3 \right] - \frac{1}{n^2}\\
&= \left[ \frac{1}{n^2} + \frac{3}{n^2} \sum_{i=1}^n a_i + 
\frac{3}{n} \sum_{i=1}^n a_i^2 + \sum_{i=1}^n a_i^3 \right] - \frac{1}{n^2}\\
&= \frac{3}{n^2} \sum_{i=1}^n a_i + \frac{3}{n} \sum_{i=1}^n a_i^2 + \sum_{i=1}^n a_i^3\\
&= \frac{3}{n} \sum_{i=1}^n a_i^2 + \sum_{i=1}^n a_i^3\\
&\leq  \frac{3}{n} \|a\|_2^2 + \|a\|_3^3 \leq  \frac{3}{n} \|a\|_2^2 + \|a\|_2^3=\frac{3}{n} (\alpha/n) + (\alpha/n)^{3/2} \;.
\end{align*}

We thus get:
\begin{align*}
\frac{256 (\alpha/\new{2})^2}{\eps^4 n (\|p\|_3^3 - \|p\|_2^4)} &\geq \frac{256 (\alpha/\new{2})^2}{\alpha^2 n (\frac{3}{n} (\alpha/n) + (\alpha/n)^{3/2})}\\
 &\geq \frac{\new{64} }{ n (\frac{3}{n} (\alpha/n) + (\alpha/n)^{3/2})}\\
&\geq \frac{\new{64} }{ \frac{3\alpha}{n} + \frac{\alpha^{3/2}}{\sqrt{n}}}\\
&\geq \frac{\new{64}n}{3\alpha+\alpha^{3/2}\sqrt{n}}\\
&\geq \min\left\lbrace \frac{\new{64}n}{3\alpha},\frac{\new{64}\sqrt{n}}{\alpha^{3/2}} \right\rbrace   \\
&\geq \frac{\new{64}\sqrt{n}}{3\alpha} \;.
\end{align*}
Therefore, any $\ell\leq \frac{\new{64}\sqrt{n}}{3\alpha}$ satisfies the conditions of Lemma~\ref{lem:samples}.
%
Combining the above, we can see that our tester works for any value of $\ell$ that is less than the sample complexity of the problem in the classical (non-distributed) model. 
}
\end{proof}


The correctness and error probability of the algorithm is established by Lemmas~\ref{lem:case1} and~\ref{lem:case2}.
This completes the proof of Theorem~\ref{thm:unif-many-sample}.

\section{Communication and Memory Lower Bounds for Uniformity Testing} \label{sec:unif-lb}

In this section, we prove our memory and communication lower bounds.

\subsection{Background from Information Theory} \label{ssec:it-background}

For completeness, we start by recalling basic definitions from information theory. 
We will first define the entropy of a random variable:
\begin{definition}
Let $X$ be a discrete random variable supported on $\{x_1,\dots , x_n\}$ that has a probability mass function $p=(p_1,\dots,p_n)$ such that $p_i=\Prob[X=x_i]$. Then we define the entropy of $X$ to be 
$H(X)=\sum_{i=1}^n p_i\log (1/p_i)$.
For the special case of $n=2$, which corresponds to a Bernoulli random variable with parameter $p\in [0,1]$, we define the binary entropy to be the following function: $H_2(p)=-p\log p - (1-p)\log(1-p)$.
\end{definition}

The entropy is a measure of randomness for a random variable. In other words, it is the number of bits of information that we get on average by observing the outcome of the random variable.

In some cases, we would like to know how much excess information 
we get by observing the outcome of a random variable $Y$ given 
that we know the outcome of another random variable $X$. 
This is usually called conditional entropy of $Y$ given $X$, 
and is defined as follows:

\begin{definition}
Let $X,Y$ be a discrete random variables supported on the sets $\mathcal{X}$ and $\mathcal{Y}$ respectively. Also let $p(x,y)$ be the joint probability mass function of $X,Y$ such that $p(x,y)=\Prob[X=x,Y=y]$. Then we define the conditional entropy of $Y$ given $X$ to be:
\[
H(Y|X) = H(X,Y) - H(X) =-\sum_{x\in\mathcal{X},y\in\mathcal{Y}} p(x,y)\log \frac{p(x,y)}{p(x)}.
\]
\end{definition}

Furthermore, the amount of information that is shared between two random variables 
is called \emph{mutual information} and defined as follows:
\begin{definition}
Let $X,Y$ be a discrete random variables. The mutual information between $X$ and $Y$ is defined as
$I(X;Y)=H(X)-H(X|Y)=H(X)+H(Y)-H(X,Y)$.
Note that this quantity is symmetric, i.e., $I(X;Y)=I(Y;X)$. We also define the conditional shared information as
$I(X;Y|Z) = H(X|Z) - H(X|Y,Z)$.
\end{definition}

The following well known lemma in information theory, intuitively implies that \new{the mutual information between two random variables $X$ and $Y$ cannot be increased by transforming $Y$ into a new variable $Z$ either deterministically or by using randomness that is independent of $X$ (i.e., without using any additional knowledge for $X$). 

\begin{lemma}[Data Processing Inequality]
\label{lm:data_proc}
Let $X,Y,Z$ be random variables, such that $X\perp Z|Y$. Then
$I(X;Z) \leq I(X;Y)$.
\end{lemma}
}
We also make use of another standard lemma known as the chain rule.
\begin{lemma}[Chain Rule]
For variables $X,Y$ and $Z$ we have that
$I(X;Y,Z) = I(X;Z) + I(X;Y|Z)$.
\end{lemma}

Finally, we use the following well known Taylor series for the binary entropy function:
\begin{fact}{\label{fact:taylor}}
$$
1{-}H_2\Big{(}\frac{1}{2}{+}a\Big{)} = \frac{1}{2\ln 2}\sum\limits_{l=1}^{\infty}\frac{(2a)^{2l}}{l(2l{-}1)} = O(a^2).
$$
\end{fact}

\subsection{Memory Lower Bounds for Uniformity Testing} \label{sec:mem-unif-lb}

In this section, we prove our memory lower bounds \new{as described by the following theorem:
\begin{theorem}\label{thm:memory_lb}
Let $\mathcal{A}$ be an algorithm which tests if a distribution $p$ is uniform versus $\eps$-far from uniform 
with error probability $1/3$, can access the samples in a single-pass streaming fashion using $m$ bits of memory 
and $k$ samples, then $k\cdot m =\Omega(\frac{n}{\eps^2})$. Furthermore, if $k<n^{9/10}$ and $m\geq k^2/n^{0.9}$, 
then $k\cdot m = \Omega(\frac{n\log n}{\eps^4})$.
\end{theorem}
\begin{remark}
This result should hold (with worse constants) if the above bounds on $m$ and $k$ are replaced by any bounds 
of the form $m\geq k^2/n^{1-c}$ and $k\leq n^{1-c}$ for any constant $c>0$.
\end{remark}

}

In this rest of this section, we prove Theorem~\ref{thm:memory_lb}.
To do so, we use the \textit{adversary method}. 
Let $X$ be a uniformly random bit. 
Based on $X$, the adversary chooses the distribution $p$ on $[2n]$ bins as follows:
\begin{itemize}
\item $X = 0$: Pick $p = U_{2n}$.
\item $X = 1$: Pair the bins as $\{1,2\},\{3,4\},\ldots,\{2n{-}1,2n\}$. Now on each pair $\{2i{-}1,2i\}$ pick a random $Y_i\in\{\pm 1\}$ to pick:
$$
(p_{2i{-}1},p_{2i}) = \begin{cases}
 (\frac{1+\eps}{2n},\frac{1-\eps}{2n}), \quad Y_i = 1\\
 (\frac{1-\eps}{2n},\frac{1+\eps}{2n}), \quad Y_i = -1
\end{cases}
$$
\end{itemize}
In either case, we can think of the output of $p$ 
as being a pair $(C,V)$, where $C$ is an element of $[n]$ is chosen uniformly, 
and $V\in \{0,1\}$ is a fair coin if $X=0$ and has bias $\eps\cdot Y_C $ if $X=1$.

Let $s_1,\ldots,s_k$ be the observed samples from $p$. 
Let $M_t$ denote the bits stored in the memory after the algorithm sees the $t$-th sample $s_t$.

Since the algorithm learns $X$ with probability $2/3$ after viewing $k$ samples, we know that $I(X;M_k) > \Omega(1)$.
On the other hand, \new{ $M_t$ is computed from $(M_{t-1},s_t)$ without using any information about $X$ \footnote{ \new{Note that we can use deterministic operations and possibly random bits, which however cannot be correlated with the random variable $X$ since $s_t$ is by definition the only sample from the distribution that is drawn between the memory states $M_{t-1}$ and $M_t$.}}. More formally, $X\perp M_t |(M_{t-1},s_t)$ and therefore we can use} the \textit{data processing inequality} \new{(Lemma \ref{lm:data_proc})} to get:
$$
I(X;M_t) \le I(X;M_{t{-}1},s_t) = I(X;M_{t{-}1}) + I(X;s_t|M_{t{-}1}).
$$
Our basic technique will be to bound $I(X;s_t|M_{t{-}1})$. This will give us an upper bound on $I(X;M_k)$ via telescoping.

The sample $s$ corresponds to which pair of bins was picked and within that pair which one of two bins was picked, that is $s= (C,V)$. $V$ is a random variable taking values in $\{0,1\}$.

Since irrespective of $X$, $C$ is uniform over the pairs of bins, 
we note that $C$ is independent of $X$ even when conditioned on the memory $M$.

Thus, $$I(X;s_t|M_{t{-}1}) = I(X;C_tV_t|M_{t{-}1})=I(X;V_t|M_{t{-}1}C_t) \;.$$


Let $\new{\alpha_{t{-}1}=\Prob[X{=}1|M_{t{-}1}C_t]}$ and thus $\Prob[X{=}0|M_{t{-}1}C_t]=1{-}\alpha_{t{-}1}$. 

We have that
\begin{align*}
\Prob[V_t=0|X=0,M_{t{-}1},C_t] &= \frac{1}{2}\\
\Prob[V_t=0|X=1,M_{t{-}1},C_t] &= \new{\frac{1{+}\eps \E[Y_{C_t}|M_{t{-}1}]}{2}}\\
\Prob[V_t=0|M_{t{-}1},C_t] &= (1{-}\alpha_{t{-}1})\frac{1}{2}{+}\alpha_{t{-}1}\new{\frac{1{+}\eps \E[Y_{C_t}|M_{t{-}1}]}{2}} = \frac{1}{2}{+}\new{\frac{\alpha_{t{-}1}\eps\E[Y_{C_t}|M_{t{-}1}]}{2}} \;.
\end{align*}

We can calculate
\begin{align*}
&I(X;V_t|M_{t{-}1}C_t) = H(V_t|M_{t{-}1}C_t)-H(V_t|M_{t{-}1}C_tX)\\
&\new{=H_2(\Prob[V_t=0|M_{t{-}1},C_t])-\{\Prob[X{=}1|M_{t{-}1}C_t]H_2(\Prob[V_t=0|X=1,M_{t{-}1},C_t]) }\\
&\new{\quad + \Prob[X{=}0|M_{t{-}1}C_t] H_2(\Prob[V_t=0|X=0,M_{t{-}1},C_t]) \}  }\\
&= H_2\Big{(}\frac{1}{2}{+}\new{\frac{\alpha_{t{-}1}\eps\E[Y_{C_t}|M_{t{-}1}]}{2}}\Big{)}-\alpha_{t{-}1}H_2\Big{(}\frac{1}{2}{+}\new{\frac{\eps\E[Y_{C_t}|M_{t{-}1}]}{2}}\Big{)}-(1{-}\alpha_{t{-}1})H_{2}\Big{(}\frac{1}{2}\Big{)}\\
&= \alpha_{t{-}1}\Big{[}1-H_2\Big{(}\frac{1}{2}{+}\new{\frac{\eps\E[Y_{C_t}|M_{t{-}1}]}{2}}\Big{)}\Big{]}-\Big{[}1-H_2\Big{(}\frac{1}{2}{+}\new{\frac{\alpha_{t{-}1}\eps\E[Y_{C_t}|M_{t{-}1}]}{2}}\Big{)}\Big{]} \;.
\end{align*}

Thus, using Fact \ref{fact:taylor} we have,
\begin{align*}
I(X;V_t|M_{t{-}1}C_t) &= \Theta(1)\alpha_{t{-}1}(1{-}\alpha_{t{-}1})\eps^2\E[Y_{C_t}|M_{t{-}1}]^2\\
&\le O(1)\eps^2\E[Y_{C_t}|M_{t{-}1}]^2.
\end{align*}

Since $C_t$ is uniformly random, we have that
\new{
\begin{equation}\label{eq:mutual_info_bound_1}
I(X;V_t|M_{t{-}1}C_t) = \frac1n\cdot \sum_{j=1}^n O(1)\eps^2\E[Y_{j}|M_{t{-}1}]^2.
\end{equation}
}
We begin by proving a relatively straightforward unconditional bound on this sum using the fact that $M_{t-1}$ has only $m$ bits of information.
\begin{lemma}\label{easyExpBoundLem}
We have that
$\sum_{j=1}^n\E[Y_{j}|M_{t{-}1}]^2 = O(m).$
\end{lemma}
\begin{proof}
First we notice that since $H(M_{t-1})\leq m$ that $I(Y_1\ldots Y_n;M_{t-1}) \leq m$, and thus that $H(Y_1\ldots Y_n|M_{t-1}) = H(Y_1\ldots Y_n) - I(Y_1\ldots Y_n;M_{t-1}) \geq n-m$. On the other hand, we have that
$$
\sum_{i=1}^n H(Y_i|M_{t-1}) \geq H(Y_1\ldots Y_n|M_{t-1}) \geq n-m.
$$
Thus,
$$
m \geq \sum_{i=1}^n [1-H(Y_i|M_{t-1})] = \Theta\left(\sum_{i=1}^n \E[Y_i|M_{t-1}]^2\right) \;,
$$
\new{where the equality comes from Fact \ref{fact:taylor} and the fact that if $\Prob [Y_i=1|M_{t-1}]=\frac12+\alpha$, then $\E[Y_i|M_{t-1}]=\Prob[Y_i=1|M_{t-1}](+1)+\Prob[Y_i=-1|M_{t-1}](-1)=(\frac12+\alpha)-(\frac12-\alpha)=2\alpha$.}
This completes our proof.
\end{proof}

Lemma \ref{easyExpBoundLem} will be enough to prove our weaker lower bound. To get the stronger one we will need a more in depth analysis. In particular, we let $r_j=\E[\#\{1\leq i \leq t-1: C_i=j \}|M_{t-1}]$. We first show that $\E[Y_j|M_{t-1}]=O(\eps r_j)$ (see Claim \ref{claim_lb} below). This leaves us with the task of bounding $\|r\|_2$. For this, we note that if $w=r/\|r\|_2$, then $r$ is only going to be large if, conditioned on $M_{t-1}$, many of the $C_i$ will lie on components where $w$ is large. However the sum of $w_{C_i}$ is a sum of independent random variables, so by an appropriate Chernoff bound, we can show that it is likely not too much larger than its mean. However, since $M_{t-1}$ only encodes $m$ bits of information, it can only correspond to an event whose likelihood is exponentially small in $m$, and this will give us our bound on $\|r\|_2$.

We will now show the following claim:

\begin{claim}\label{claim_lb}
We have that $|\E[Y_{j}|M_{t{-}1}]|=O(\eps \cdot r_j)$.
\end{claim}
\begin{proof}
It suffices to show that
$$
|\E[Y_{j}|s_1,\ldots,s_{t-1},X]|=O\left( \eps\#\{1\leq i \leq t-1: C_i=j \}\right),
$$
as our final result will follow by averaging over the $s_i$ and $X$ conditioned on $M_{t-1}$.

If $X=0$, this is trivial since in this case the $s_i$ convey no information about $Y_j$ so the expectation of $Y_j$ is $0$.

If $X=1$, it is not hard to see by Bayes' Theorem that if $n_1$ is the number of times when $s_i = (j,0)$ and $n_2$ the number of times $s_i = (j,1)$, then the expectation of $Y_j$ conditioned on $X$ and the $s_i$ is
$$
\frac{(1+\eps)^{n_1}(1-\eps)^{n_2}-(1-\eps)^{n_1}(1+\eps)^{n_2}}{(1+\eps)^{n_1}(1-\eps)^{n_2}+(1-\eps)^{n_1}(1+\eps)^{n_2}} = O(\eps(n_1+n_2)) \;,
$$
and our result follows.
\end{proof}

Since $C_t$ is uniform independent of $M_{t{-}1}$, we have that 
$$\E[Y_{C_t}|M_{t{-}1}]^2 = O\left(\eps^2 \frac1n \sum_{j=1}^n r_j^2\right)=O\left(\frac{\eps^2 \Vert r \Vert_2^2}{n}\right) \;.$$
Therefore, we get that:
\begin{equation}
\label{eq:mutual_info_bound}
I(X;V_t|M_{t{-}1}C_t)\leq O(1) \frac{\eps^4 \Vert r \Vert_2^2}{n} \;.
\end{equation}

\paragraph{Typical Memory States} Consider a fixed algorithm $\mathcal{A}$. 
Call a memory state $M$ \emph{typical} for time step $t$ if the following hold:
\begin{itemize}
\item $\Prob(M_t = M) > e^{-m}$.
\item The corresponding vector $r$ has $\|r\|_\infty \leq 30$.
\end{itemize}
We will need the first condition to ensure that $M_t$ does not encode events that are too unlikely, and we will need the second to bound the maximum size of individual contributions for our Chernoff bound. Fortunately, both of these events happen with high probability as we see below:




\begin{claim}\label{typicalClaim}
Assuming $t\leq n^{9/10}$ and $m$ is bigger than a sufficiently large multiple of $\log(n)$, 
we have that $M_t$ is typical for time $t$ with probability at least $1-1/n$.
\end{claim}
\begin{proof}
First, we deal with the probability that $M_t$ violates the first condition, in particular that it is a transcript that shows up with probability at most $e^{-m}$. For this, we note that there are at most $2^m$ such transcripts, each occurring with probability at most $e^{-m}$, and so the total probability that any of them occur is at most $(2/e)^m < 1/(2n)$.

Next, we deal with the probability that $M_t$ violates the second condition. In particular, we bound the probability that there exists a $j$ so that the expected number of $C_i$ equal to $j$ for $1\leq i\leq t$ conditioned on $M_t$ is at least $30$. For this we, note that for any particular $j$, \new{the expectation of $\max(0,\#\{1\leq i\leq t: C_i=j\}-20)$ over sets of samples is at most $n^{-2}$.} Therefore, the expectation over transcripts of $\max(0,r_j-20)$ is at most $n^{-2}$. Our result now follows by a Markov inequality and union bound over $j$.
\end{proof}

We are now ready to prove an upper bound on $N=\|r\|_2^2$ in the following lemma:

%
%
%

\begin{lemma}\label{lm:rbound}
For the fixed transcript $A$ typical for time $t$, with $m\geq t^2/n^{0.9}$, we have
$$
N= \|r\|_2 = O\left(\sqrt{\frac{m}{\log n}}\right) \;.
$$
\end{lemma}
\begin{proof}
Let $w=r/N$. Note that $\|r\|_2 = r\cdot w$, and that $\|w\|_2=1$.

Let $X_\ell = w_{C_\ell}~ 1\le \ell \le t$. These are i.i.d random variables taking values in $[0,\frac{30}{N}]$ with mean $\E[X_\ell]=\frac1n\sum\limits_{i=1}^n w_i\le \frac{1}{\sqrt{n}}$, since $\|w\|_2=1$.

Define $X := \sum\limits_{\ell=1}^{t}w_{C_\ell}$ and note that $\mu = \E[X] \le \frac{t}{\sqrt{n}}$.
We have that
\begin{align*}
N&=\E\Big{[}\sum\limits_{\ell=1}^{t}w_{C_\ell}|M_t=A\Big{]} = \frac{1}{\Prob(M_t{=}A)}\int\limits_{a=0}^{\infty}\Prob[X>a, M_t{=}A]da\\
&= \frac{1}{\Prob[M_t{=}A]}\int\limits_{a=0}^{N/2}\Prob[X>a, M_t{=}A]da + \frac{1}{\Prob[M_t{=}A]}\int\limits_{a = N/2}^{\infty}\Prob[X>a, M_t{=}A]da\\
&\le \frac{1}{\Prob[M_t{=}A]}\int\limits_{a=0}^{N/2}\Prob[M_t{=}A]da + \frac{1}{\Prob[M_t{=}A]}\int\limits_{a = N/2}^{\infty}\Prob[X>a]da\\
&= \frac{N}{2} + \frac{1}{\Prob[M_t{=}A]}\int\limits_{a = N/2}^{\infty}\Prob[X>a]da \;.
\end{align*}
Thus, we have
$$
\frac{N}{2} \le e^m\int\limits_{a = N/2}^{\infty}\Prob[X>a]da \;.
$$
Now let us bound the tail $\Prob[X>a]$. We have $\Prob[X>a] = \Prob[Nx>Na]$. 
We would like to show that $N \le\sqrt{\frac{m}{\log n}}$. 
Thus, we can assume that $N > 4\frac{t}{n^{0.49}}$ else 
we already have $N \leq \sqrt{\frac{m}{\log n}}$. Hence, 
we can assume that $a > 2\frac{t}{n^{0.49}}$ in the above integral. 
We write $a = (1+\delta)\mu$ and apply the Chernoff bound 
on the random variable $\frac{N}{30}\cdot X$ 
(note that this is a sum of i.i.d random variables supported in $[0,1]$) to get:
$$
\Prob[X>a]=\Prob[\frac{N}{30}X>\frac{N}{30}(1+\delta)\mu] 
< \frac{e^{\frac{\delta N\mu}{30}}}{(1+\delta)^{(1+\delta)N\mu/30}} 
\le e^{-\frac{1}{40}Na\log(1+\delta)} \;.
$$
We have $1+\delta = \frac{a}{\mu} > \frac{N\sqrt{n}}{30t} > n^{1/200}$. 
Thus, for $a \ge \frac{N}{2}$ we have
$$
\Pr(X>a) \le e^{-\alpha Na\log n} \;,
$$
for some constant $\alpha>0$.
Substituting in the above integral gives:
\begin{align*}
\frac{N}{2} &\le e^m\int\limits_{a = N/2}^{\infty}\Prob[X>a]da \le e^m\int\limits_{a = N/2}^{\infty}e^{-\alpha Na\log n}da = \frac{1}{\alpha N\log n}e^{m-\alpha N^2\log n/2} \;.
\end{align*}
Thus, we have for some constant $\alpha$:
$$
\frac{\alpha N^2\log n}{2} \le e^{m-\alpha N^2\log n/2} \;.
$$
Since $m \ge 1$, the equation $\theta \le e^{m-\theta}$ can have a solution 
only when $\theta \le m$. That is $\frac{\alpha N^2\log n}{2} \le m$, 
this gives us the required bound $\|r\|_2 = N \le \sqrt{2/\alpha}\sqrt{\frac{m}{\log n}} = O(\sqrt{\frac{m}{\log n}})$.
\end{proof}

%
%
%

\begin{proof}[Proof of Theorem~\ref{thm:memory_lb}]
Using Equation \eqref{eq:mutual_info_bound_1} and Lemma \ref{easyExpBoundLem}, we get that 
\[
I(X;V_t|M_{t{-}1}C_t) \leq O(1)\frac{\eps^2m}{n} \;.
\]
However, we know that:
 \begin{align*}
\Omega(1) \le I(M_k;X) &= \sum\limits_{t=0}^{k\new{-1}} I(M_{t{+}1};X) - I(M_t;X)\\
&= \sum\limits_{t=0}^{k-1} I(M_t,S_{t{+}1};X)-I(M_t;X)\\
&=\sum\limits_{t=0}^{k-1} I(S_{t+1};X|M_t)\\
&=\sum\limits_{t=0}^{k-1} I(V_{t+1};X|M_t,C_{t+1})\\
&=O(1)\frac{k\eps^2 m}{n}.
\end{align*}
This implies that $k\cdot m = \Omega(\frac{n}{\eps^2})$.

Under our stronger assumptions, we can instead use Lemma \ref{lm:rbound} to similarly obtain:
\begin{align*}
\Omega(1) \le I(M_k;X) &=\sum\limits_{t=0}^{k-1} I(V_{t+1};X|M_t,C_{t+1})\\
&=O(1)\frac{k\eps^4m}{n\log n}+O(k/n).
\end{align*}
The last line here comes from the fact that $I(V_{t+1};X|M_t,C_{t+1}) = O(\eps^4 m /n\log(n))$ for typical transcripts $M_t$ and the fact that $M_t$ is typical except with probability $1/n$.

Thus, equivalently, we have $k\cdot m=\Omega(\frac{n\log n}{\eps^4})$, completing the proof.
\end{proof}

\subsection{Communication Lower Bounds for Uniformity Testing} \label{sec:com-unif-lb}

In this section, we will show a communication lower bound for distributed uniformity testing algorithms
in our one-pass communication model, via a reduction to the streaming/limited memory setting.
 In particular, we show the following theorem:

\begin{theorem}\label{thm:reduction}
Suppose that there exists a communication protocol with a transcript of length $|T|$ bits, 
for the setting where each machine holds $\ell$ samples, that can distinguish between 
a uniform distribution and one that is $\eps$-far from uniform in total variation distance. 
Then there exists a streaming algorithm that uses at most $m=|T|+\ell\cdot \log n$ bits of memory 
and has access to a stream of at most $s=|T|\cdot \ell$ samples.
\end{theorem}

\begin{proof}
We will simulate the protocol by storing in memory the entire communication 
transcript up to any given point in the simulated protocol, while having some 
additional space in order to temporarily store the samples of a single player (machine) at a time.

In particular, we consider the stream of $t\cdot \ell$ samples, where $t$ is the number of players that participate, that is created by taking the $\ell$ samples of the first player to speak and iteratively appending the $\ell$ samples of the next player to speak until there are no more players. We also use the memory to remember the communication transcript so far at any given point and the samples of the player speaking in that round so that the algorithm is able to compute the bits that \new{the players} send.

Therefore, during any given round $i$ of the communication protocol, the partial transcript $T_{i-1}$ of the communication in the first $i-1$ rounds is stored in memory along with the $\ell$ samples of the player that is about to speak in round $i$. \new{Note that, since the referee is not going to ask any questions again to that particular player, the exact samples of that player are no longer useful to the algorithm after the current round ends. Therefore, those $\ell\log n$ bits of additional memory can be reused while simulating the next round.  } However, the algorithm will use the bits transmitted by that player along with the current partial transcript $T_{i-1}$, to create the new partial transcript $T_i$.

Observe that every player has to send at least $1$ bit, since otherwise we can assume that they did not participate in the protocol. Therefore, we have that $t\leq |T|$ and consequently our stream will have at most $s=|T|\cdot \ell$ samples.

Furthermore, the transcript that is created after the last player speaks is the one that contains the most information among the partial transcripts which are all optimally compressed. Thus, we have that $\forall i: |T_i|\leq |T|$, where $T=T_t$ is the transcript of the entire communication. This means that no more than $m=|T|+\ell\cdot \log n$ bits of memory are needed at any given point of the execution.
\end{proof}

Using the above theorem, we can prove the following two corollaries:

\begin{corollary} \label{cor:com-unif-lb1}
Let $\pi$ be a distributed communication protocol, for the setting 
where each machine holds $\ell$ samples, which tests if a distribution $p$ is uniform versus $\eps$-far from uniform 
with error probability $1/3$, and the referee asks questions to each player only once. 
\new{Then, $\pi$ must involve $\Omega\left(\frac{\sqrt{n/\ell}}{\eps}\right)$ bits of communication 
for any $\ell=O\left( \frac{n^{1/3}}{\eps^{4/3}(\log n)^{1/3}}\right)$.  Furthermore,} $\pi$ must involve 
$\Omega\left(\frac{\sqrt{n/\ell}}{\eps^2}\sqrt{\log n}\right)$ bits of communication 
for any $\ell=O\left( \eps^{4/3} n^{0.3} \right)$.
\end{corollary}
\begin{proof}
Suppose, for the sake of contradiction, that there exists such a protocol that uses $t=\Theta\left(\frac{\sqrt{n/\ell}}{\eps^2}\sqrt{\log n}\right)$ bits of communication with a sufficiently small implied constant. Then, using theorem \ref{thm:reduction}, we conclude that there also exists an streaming algorithm that uses a stream of size $s\leq t\ell$ samples and a memory of size $m=\Theta(t+\ell\log(n))=\Theta(t)$ bits. Furthermore, we have that
$$s^2/(mn^{0.99}) = \Theta(t\ell^2 / n^{0.99}) \ll (\ell^{3/2}/(n^{0.45}\eps^2)) \ll 1.$$
Therefore, Theorem \ref{thm:memory_lb} applies and we must have $m s = \Omega(n\log(n)/\eps^4)$, but $ms=O(t^2\ell)$, which is a sufficiently small multiple of $n\log(n)/\eps^4$, that it yields a contradiction.  

\new{Note that for any $\ell=O\left( \frac{n^{1/3}}{\eps^{4/3}(\log n)^{1/3}}\right)$, it still holds that $m=\Theta(t+\ell\log(n))=\Theta(t)$. We will combine now Theorem \ref{thm:reduction} with the weaker version of Theorem \ref{thm:memory_lb}, and assume that there exists such a protocol that uses $t^\prime=\Theta\left(\frac{\sqrt{n/\ell}}{\eps}\right)$ bits of communication with a sufficiently small implied constant. In this case, we must have $m s = \Omega(n/\eps^2)$, but $ms=O((t^\prime)^2\ell)$, which is a sufficiently small multiple of $n/\eps^2$, that it yields a contradiction. }
\end{proof}

Furthermore, if we have a restricted number of samples, we can get better communication lower bounds:
\begin{corollary}\label{cor:com-sample-tradeoff}
Let $\pi$ be a distributed communication protocol, for the setting where each machine holds $\ell$ samples 
with a total of $t$ machines, which tests if a distribution $p$ is uniform versus $\eps$-far from uniform with error probability $1/3$, 
and the referee asks questions to each player only once. Then, if $t=O\left(\frac{\sqrt{n/\ell}}{\eps^2}\sqrt{\log n}\right)$ 
and $t\ell = O(n^{0.6}/\eps^{4/3})$, $\pi$ must involve $\Omega(\frac{n\log(n)}{\eps^4 t\ell})$ bits of communication. 
\end{corollary}
\begin{proof}
Again we use Theorem \ref{thm:reduction}. We now have a streaming algorithm using $k=t\ell$ samples and $m=|\pi|+\ell\log(n)$ memory. We claim that this is impossible even if $|\pi|=p=\Theta(\frac{n\log(n)}{\eps^4 t\ell})$ with the implied constant sufficiently small. In fact, this in case we have that $m=O(p)$. We have that $mn^{0.9} k = \Theta(n^{1.9}\log(n)/\eps^4) > k^3$, and so the strong version of Theorem \ref{thm:memory_lb} applies. This means that $mk = \Omega(n\log(n)/\eps^4)$, when in reality it is too small a constant times this, yielding a contradiction. 
\end{proof}

A case of particular interest for the above is when $t\ell = O(\sqrt{n}/\eps^2)$ is the information-theoretically optimal number of samples. In this case (so long as $\ell = O(\log(n))$) our communication must be at least $\Omega(\sqrt{n}\log(n)/\eps^2)$, which is not much better than sending all of the samples directly.

\new{
Finally, the following corollary gives a communication complexity lower bound for all values of $\\ \Omega\left( \frac{n^{1/3}}{\eps^{4/3}(\log n)^{1/3}}\right)\leq\ell\leq O\left(\frac{\sqrt{n}}{\eps^2}\right)$ using the weaker version of Theorem \ref{thm:memory_lb}. 
\begin{corollary}
Let $\pi$ be a distributed communication protocol, for the setting where each machine holds $\ell$ samples, 
which tests if a distribution $p$ is uniform versus $\eps$-far from uniform with error probability $1/3$, 
and the referee asks questions to each player only once. Then, $\pi$ must involve $\Omega(\frac{n}{\ell^2\eps^2\log n})$ 
bits of communication for any $\Omega\left( \frac{n^{1/3}}{\eps^{4/3}(\log n)^{1/3}}\right)\leq\ell\leq O\left(\frac{\sqrt{n}}{\eps^2}\right)$.  
\end{corollary} 
\begin{proof}
Suppose, for the sake of contradiction, that there exists such a protocol that uses $t=\Theta\left(\frac{n}{\ell^2\eps^2\log n}\right)$ bits of communication with a sufficiently small implied constant. Then, using Theorem \ref{thm:reduction}, we conclude that there also exists an streaming algorithm that uses a stream of size $s$ samples and a memory of size $m$ bits, such that $ m\cdot k=\Theta( t\ell^2\log n) $, since $k=t\ell$ and $m=\Theta(\ell \log n)$ for this range of values for $\ell$.

This means that $mk = \Omega(t\ell^2\log n)=\Omega(\frac{n}{\eps^2})$, when in reality it is too small a constant times this, yielding a contradiction due to Theorem \ref{thm:memory_lb}.
\end{proof}
}

\section{Communication and Memory Efficient Closeness Testing} \label{sec:closeness-alg}

In this section, we design our protocols for closeness testing.
We start with the setting of memory and then give our communication efficient protocols.

\subsection{Memory Efficient Closeness Testing} \label{ssec:close-memory}

In this section, we provide an algorithm for closeness testing in the streaming model 
that uses $O\left(\frac{n}{\sqrt{m}\epsilon^2}\right)$ samples and $O(m\log(n))$ \new{bits of} memory 
\new{where $m$ is a parameter such that} $$\min(n, n^{2/3}/\eps^{4/3}) \gg m \gg 1 \;.$$ 
By reparametrizing, this implies an algorithm 
with $\min(n\log(n),n^{2/3}\log(n)/\eps^{4/3}) \gg m\gg \log(n)$ \new{bits of} memory and 
$O\left(\frac{n\sqrt{\log(n)}}{\sqrt{m}\epsilon^2}\right)$ samples. 
However, \new{we are going to use the former parametrization assuming an upper bound of $O(m)$ words of memory (each of length $O(\log n)$ bits), as it} will be more convenient for us, so we will use that.

\new{The performance of the algorithm is described in the following theorem:
\begin{theorem}\label{thm:closeness_mem}
Let $p,q$ be two discrete distributions on $[n]$. Suppose that $\min(n, n^{2/3}/\eps^{4/3}) \gg m \gg 1$. Then there exists a single pass streaming algorithm that uses at most $m\log n$ bits of memory and $O(\frac{n}{\sqrt{m}\eps^2})$ samples from $p$ and $q$, and distinguishes between the cases that $p = q$ versus $\|p-q\|_1\ge \eps$ 
with probability at least $2/3$.
\end{theorem}
}

The algorithm is given in pseudo-code below:

\medskip

\fbox{\parbox{5.8in}{

{\bf Algorithm} \textsc{Test-Closeness-Memory}$(p, q, n, m, \eps)$ \\
Input: Sample access to distributions $p, q$ over $[n]$, memory bound $m$, and $\eps>0$.\\
Output: ``YES'' if $p=q$; ``NO'' if $\|p-q\|_1 \ge \eps.$
\begin{enumerate}
    \item Draw $O(m)$ samples from $p$ and $q$ to flatten them to $p',q'$ such that $\|p'\|_{2},\|q'\|_{2} \le O\left(\frac{1}{\sqrt{m}}\right)$. Let $[n']$ be the new domain.
	
    \item Apply a hash map $h$ to $p',q'$. This hash map $h: [n] \to [m]$ 
approximately preserves $\|p'-q'\|_2$ and $\|p'\|_2$ with constant probability.
             
    \item Use a standard $\ell_2$ tester to distinguish between $h(p'){=}h(q')$ and $\|h(p'){-}h(q')\|_2 \gg \frac{\eps}{\sqrt{n}}$.
\end{enumerate}
}}

\medskip

This section is devoted to the proof of Theorem~\ref{thm:closeness_mem}.

\paragraph{Flatten $p,$ and $q$.}
Our algorithm begins by taking (and storing) $m$ samples from each of $p$ and $q$. We use these samples to produce the split distributions $p'$ and $q'$ whereby the $i^{\textrm{th}}$ bin is split into $a_i$ equal sub-bins where $a_i$ is one more than the number of copies of $i$ in this set of samples. We note the following important facts from \cite{DK16}:
\begin{itemize}
\item Given the list of samples, one can simulate a sample from $p'$ (resp. $q'$) from a single sample of $p$ (resp. $q$) and some additional randomness.
\item $\|p-q\|_1 = \|p'-q'\|_1$.
\item $\|p'\|_2, \|q'\|_2 = O(1/\sqrt{m})$ with at least $9/10$ probability.
\end{itemize}
Our analysis from here on out will be under the assumption that the last property holds.

\paragraph{Hash $p'$ and $q'$.} Our next step is to pick a hash map $h:[n']\rightarrow [m]$ (where $n'$ is the size of the domain of $p'$ and $q'$) from a $4$-wise independent family (note that for an appropriate family we can store $h$ using $O(m\log(n))$ bits). We claim that with at least $9/10$ probability that $\|h(p')\|_2,\|h(q')\|_2$ are not too big and that $\|h(p')-h(q')\|_2 \approx \|h(p)-h(q)\|_2$.

In particular, we show the following lemma:
\begin{lemma}
We have the following:
\begin{align*}
\E_{h}[\|h(p'){-}h(q')\|_2^2] &= \Big{(}1{-}\frac{1}{m}\Big{)}\|p'{-}q'\|_2^2\\
\Var_h\Big{[}\|h(p'){-}h(q')\|_2^2\Big{]} &= \frac{1}{m}\Big{(}1{-}\frac{1}{m}\Big{)}\Big{[}\|p'{-}q'\|_2^4-\|p'{-}q'\|_4^4\Big{]}\\
\E_{h}[\|h(p')\|_2^2] &= \frac{1}{m} + \Big{(}1-\frac{1}{m}\Big{)}\|p'\|_2^2\\
\Var_h\Big{[}\|h(p')\|_2^2\Big{]} &=\frac{1}{m}\Big{(}1{-}\frac{1}{m}\Big{)}\Big{[}\|p'\|_2^4-\|p'\|_4^4\Big{]}
\end{align*}
\end{lemma}
\begin{proof}
We can write:
\begin{align*}
\|h(p'){-}h(q')\|_2^2 &= \sum\limits_{i \in [m]}\Big{[}\sum\limits_{j\in [n']}(p'_j-q'_j)I\{h(j)=i\}\Big{]}^2\\
&= \sum\limits_{i \in [m]}\sum\limits_{j_1,j_2\in [n']}(p'_{j_1}-q'_{j_1})(p'_{j_2}-q'_{j_2})I\{h(j_1)=h(j_2)=i\}\\
&= \sum\limits_{j_1,j_2\in [n']}(p'_{j_1}-q'_{j_1})(p'_{j_2}-q'_{j_2})I\{h(j_1)=h(j_2)\}\\
&= \|p'-q'\|_2^2 + \sum\limits_{j_1 \ne j_2\in [n']}(p'_{j_1}-q'_{j_1})(p'_{j_2}-q'_{j_2})I\{h(j_1)=h(j_2)\} \;.
\end{align*}
We therefore have that:
\begin{align*}
\E_h[\|h(p'){-}h(q')\|_2^2] &= \|p'-q'\|_2^2 + \frac{1}{m}\sum\limits_{j_1 \ne j_2\in [n']}(p'_{j_1}-q'_{j_1})(p'_{j_2}-q'_{j_2})\\
&= \|p'-q'\|_2^2 + \frac{1}{m}\sum\limits_{j_1\in [n']}\Big{[}(p'_{j_1}-q'_{j_1})\sum\limits_{j_2\ne j_1 \in [n']}(p'_{j_2}-q'_{j_2})\Big{]}\\
&= \|p'-q'\|_2^2 - \frac{1}{m}\sum\limits_{j \in [n']}(p'_{j}-q'_{j})^2 = \Big{(}1{-}\frac{1}{m}\Big{)}\|p'{-}q'\|_2^2 \;,
\end{align*}
and
\begin{align*}
\Var_h[\|h(p'){-}h(q')\|_2^2] &= \sum\limits_{j_1 \ne j_2\in [n']}\Var_h\Big{[}(p'_{j_1}-q'_{j_1})(p'_{j_2}-q'_{j_2})I\{h(j_1)=h(j_2)\}\Big{]}\\
&= \sum\limits_{j_1 \ne j_2\in [n']}(p'_{j_1}-q'_{j_1})^2(p'_{j_2}-q'_{j_2})^2\frac{1}{m}\Big{(}1-\frac{1}{m}\Big{)}\\
&= \frac{1}{m}\Big{(}1-\frac{1}{m}\Big{)}\sum\limits_{j_1 \in [n']}(p'_{j_1}-q'_{j_1})^2\Big{[}\sum\limits_{ j_2\ne j_1\in [n']}(p'_{j_2}-q'_{j_2})^2 \Big{]}\\
&= \frac{1}{m}\Big{(}1-\frac{1}{m}\Big{)}\sum\limits_{j_1 \in [n']}(p'_{j_1}-q'_{j_1})^2\Big{[}\|p'-q'\|_2^2 - (p'_{j_1}-q'_{j_1})^2\Big{]}\\
&= \frac{1}{m}\Big{(}1{-}\frac{1}{m}\Big{)}\Big{[}\|p'{-}q'\|_2^4-\|p'{-}q'\|_4^4\Big{]} \;.
\end{align*}
The other two identities follow similarly. 
\end{proof}

By Chebyshev's inequality, we have that the following statements hold true with $90\%$ probability:
\begin{align*}
\|h(p')-h(q')\|_2^2 &\ge \frac{1}{2}\|p'{-}q'\|_2^2\\
\|h(p')\|_2^2 &\le \frac{1}{m} + \frac{3}{2}\|p'\|_2^2 = O(1/m).
\end{align*}

We now have to distinguish between a completeness case where $p=q$ and thus $h(p')=h(q')$ and a soundness case where $\|p'-q'\|_1 = \|p-q\|_1 \geq \eps$, and therefore $\|h(p')-h(q')\|_2 \gg \|p'-q'\|_2 \gg \frac{\eps}{\sqrt{n}}$.

We will also need the following lemma from \cite{CDVV14}:
\begin{lemma}{\label{l2tester}}
Let $p,q$ be distributions such that $\max(\|p\|_2,\|q\|_2) \le b$, then there is an estimator that takes $O(\frac{b}{\epsilon^2})$ samples from $p,q$ and estimates $\|p{-}q\|_2$ up to an error of $\epsilon$.
\end{lemma}

\new{
\begin{proof}[Proof of Theorem \ref{thm:closeness_mem}]
Using the tester from Lemma \ref{l2tester} for $\eps^\prime=\frac{\eps}{\sqrt{n}}$ and given that $b=1/\sqrt{m}$, we can distinguish between $h(p')=h(q')$ and $\|h(p')-h(q')\|_2 \gg \frac{\eps}{\sqrt{n}}$ in $O\left(\frac{n}{\sqrt{m}\eps^2}\right)$ samples from $h(p')$ and $h(q')$ (which can be simulated given samples from $p$ and $q$). We note also that this tester only needs to know the number of samples from each of $h(p')$ and $h(q')$ that landed in each bin, and can thus be simulated in a streaming algorithm with $O(m\log(n))$ memory by keeping a running total of the number of samples from each bin.
This establishes the correctness of our algorithm.

As far as memory usage is concerned, the algorithm uses a total of $O(m\log(n))$ bits of memory for each of its steps. In particular, these steps are: recording a set of samples for flattening, storing the seed of the hash function $h$, and storing the counts of the number of samples from $h(p'),h(q')$ from each bin. Thus, the total memory usage is $O(m\log(n))$ bits.
\end{proof}
}

\subsection{Communication Efficient Algorithm for Distributed Closeness Testing} \label{ssec:close-com}

We use a somewhat different algorithm for the communication version of this problem. 
The basic idea is that while our memory algorithm compared $h(p)$ to $h(q)$ for some hash function $h$, 
our algorithm here will compare the conditional distribution $(p|W)$ to $(q|W)$, 
for some randomly chosen subset $W$ of our domain. After applying some flattening, 
we can ensure that with high probability the difference between $p$ and $q$ on $W$ 
approximates the difference of $p$ and $q$ on the whole domain. 
Since $W$ is small, we will need fewer samples to test closeness on $W$. 
Of course, we cannot make $W$ too small, as then we will need to query too many machines 
before even finding a sample from $W$. This is balanced our when $|W|\approx n/(\ell \log(n))$, 
so that one out of every $\log(n)$ machines should have a sample (which it communicates in $\log(n)$ bits), 
while the other $\log(n)$ machines need one sample each to tell us that they have no samples from $W$.

\begin{theorem}\label{thm:closeness-New-alg}
Suppose that $\ell = O(n\eps^4/\log(n))$. Then there exists an algorithm that given distributed sample access to two distributions
$p$ and $q$ over $[n]$ distinguishes with probability $2/3$ between the cases
$p=q$ and $\|p-q\|_1>\eps$ using $O\left(\frac{n^{2/3}\log^{1/3}(n)}{\ell^{2/3}\epsilon^{4/3}}\right)$ 
bits of communication.
\end{theorem}

The algorithm is given in pseudo-code below:

\medskip

\fbox{\parbox{5.8in}{

{\bf Algorithm} \textsc{Test-Closeness-Distributed}$(p, q, n, \eps)$ \\
Input: Each player has $\ell$ samples from each of $p$ and $q$ over $[n]$, $\eps>0$.\\
Output: ``YES'' if $p=q$; ``NO'' if $\|p-q\|_1 \ge \eps.$
\begin{enumerate}
    \item Let $C$ be a sufficiently large constant. Abort the following algorithm if more than 
             $\frac{C^2(n\log n)^{2/3}}{\ell^{2/3}\epsilon^{4/3}}$ bits of communication are ever used.
	
    \item Draw $N = \frac{C\ell(\log n)}{\eps}$ samples from $p$ and $q$ to flatten them to $p'$ and $q'$. 
             Let $[n']$ be the new domain.
             
    \item The referee picks a random subset $W$ of $[n']$ by selecting each element with probability 
             $r=\frac{1}{\ell\log n}$ and broadcasts this set $W$ to all the machines.
     
     \item The referee asks $M=C\log(n)|W|^{2/3}/\eps^{4/3}$ machines first if they have any samples from $W$, 
              and if so for the list of these samples along with which distribution they are from.
    
     \item Let $m_1$ of the above samples be from $p$ and $m_2$ be from $q$. 
              Unless $|m_1-m_2| < C\sqrt{m_1}$ and $|m_1|>|W|C^{2/3}/\eps^{4/3}$ return ``NO''.
              
    \item Use the above samples to test $\eps/C^{1/2}$-closeness of two distributions on $W$ and return the result.
\end{enumerate}
}}

\medskip

This section is devoted to the proof of Theorem~\ref{thm:closeness-New-alg}.

Using the same analysis for flattening as in~\cite{DK16} (see also Section~\ref{ssec:close-memory}) 
gives us that $p',q'$ satisfy $\|p'\|_2, \|q'\|_2 \le \frac{1}{\sqrt{N}}$ and that $\|p'-q'\|_1 = \|p-q\|_1$ 
with probability at least $99\%$.  Note that $p'_{|W},q'_{|W}$ are  non-normalized pseudo-distributions 
given by restrictions of $p',q'$ to $W$, and $(p'|W)$ and $(q'|W)$ (the corresponding conditional distributions) 
are their normalized distributions.

We also note that
\begin{align*}
\E[\|p'-q'\|_2^2] & = \sum_i |p_i-q_i|^2\E[1/(a_i+1)]\\ 
& = O\left(\sum_i |p_i-q_i|^2/(N(p_i+q_i)) \right) = O(1/N)\sum_i |p_i-q_i| = O(\|p-q\|_1/N).
\end{align*}
Therefore, $\|p'-q'\|_2^2 = O(\|p-q\|_1/N)$ with $99\%$ probability. 
We assume this and the above bounds on $\|p'\|_2$ and $\|q'\|_2$ throughout the rest.

Next, we analyze the sizes of $W,p(W),q(W)$ and $\|p|_W - q|_W\|_1$. 
In particular, we note that since $W$ selects each element independently 
with probability $1/(\ell\log(n))$, 
$|W|$ has mean $n'/(\ell \log(n))$ with a similar variance, 
and so $|W| = \Theta(n/(\ell\log(n)))$ with $99\%$ probability (note that $n'=n+N=\Theta(n)$). 
The mean of $p'(W)=1/(\ell\log(n))$ and the variance is $\|p'\|_2^2r$, therefore with $99\%$ probability 
$p'(W) = \Theta(nr),$ and similarly for $q'(W)$. Finally, we have that $\|p'|_W-q|_W\|_1$ has mean 
$\|p'-q'\|_1r = \|p-q\|_1r$ and variance $\|p'-q'\|_2^2 r = O(\|p-q\|_1 r/ N)$. So by Chebyshev's inequality, 
with $99\%$ probability we have that if either $p=q$ or $\|p-q\|_1\geq \eps$, then 
$\|p'|_W-q'|_W\|_1 = \Theta(\|p-q\|_1 r).$

We next consider the completeness and soundness cases, ignoring the possibility of the early abort.

\paragraph{Completeness} If $p=q$, then $p'(W)=q'(W)$ and $(p|W)=(q|W)$. 
We note that $m_1$ and $m_2$ each have average values $p'(W)M\ell = \Theta(r\ell M) = \Theta(M/\log(n))$ 
and variances less than their means. This implies that with at least $99\%$ probability it holds 
$|m_1-m_2| < C\sqrt{m_1}$ and $|m_1|>|W|C^{2/3}/\eps^{4/3}$. 
Additionally, since $(p'|W)=(q'|W)$, we will pass the closeness test for these distributions.

\paragraph{Soundness} If $\|p-q\|_1 > \eps$, we have that $\|p'|_W-q'|_W\|_1 = \Omega(\eps r).$ 
We note that this implies that either $|p'(W)-q'(W)| > \eps r / C^{1/3}$ or 
$\|(p'|W)-(q'|W)\|_1 > \eps/ C^{1/2}$. This is because
\begin{align*}
\|p'|_W-q'|_W\|_1 & = \|(p'|W)p(W) - (q'|W)q(W)\|_1 \\ & \leq \|(p'|W)p(W) - (q'|W)p(W)\|_1 + \|(q'|W)p(W) - (q'|W)q(W)\|_1\\
& = p(W)\|(p'|W)-(q'|W)\|_1 + |p'(W)-q'(W)| \;.
\end{align*}
First, consider what happens if $|p'(W)-q'(W)| > \eps r / C^{1/3}$. We notice that $m_1$ and $m_2$ have means of $M \ell p'(W)=\Theta(M/\log(n))$ and $M \ell q'(W)=\Theta(M/\log(n))$, respectively, with variances on the order of their means. Now if $|p'(W)-q'(W)| > \eps r / C^{1/3}$, the means of $m_1$ and $m_2$ will differ by $\Omega(M\eps/\log(n)C^{1/3})=\Omega(C^{2/3}n^{2/3}/(\ell^{2/3}\eps^{1/3}))$, while the variance is $O(Cn^{2/3}/(\ell^{2/3}\eps^{4/3})).$ Since the difference of the means is much bigger than both the square root of the mean of $m_1$ and the square root of the variance, $|m_1-m_2|$ will be bigger than $C\sqrt{m_1}$ with $99\%$ probability.

On the other hand. if $\|(p'|W)-(q'|W)\|_1 > \eps/ C^{1/2}$, our closeness tester in the last step will fail.

\paragraph{Communication Complexity} Here we show that the communication complexity 
of the algorithm is within the desired bounds, and that we have enough samples to perform the test in the last step. 
Firstly, we note that the $N$ samples in the first step requires only $N\log(n)$ communication, 
which is well within our bounds.  The other major step requires asking $M$ machines. 
It takes only $O(M)$ communication for each machine to report whether or not they have a sample, 
and we have an average of $M\ell(p'(W)+q'(W))$ samples that take $O(\log(n))$ bits each. 
This is at most $O(M\ell \log(n) r) = O(C|W|^{2/3}/\eps^{4/3}) = O(Cn^{2/3}/(\ell^{2/3}\log^{2/3}(n)\eps^{4/3})) $ 
samples, for an appropriate number of bits.

Finally, we note that for the last step since $|W|=n/(\ell\log(n)) \gg \eps^{-4}$, 
our tester only requires $O(|W|^{2/3}/\eps^{4/3})$ samples, which are available.
This completes the proof of Theorem~\ref{thm:closeness-New-alg}.

\section{Conclusions and Future Directions} \label{sec:conc}

This work gave algorithms and lower bounds, in some cases matching,
for distribution testing with communication and memory constraints. 
Our work is a first step in these directions and suggests a host
of interesting open problems. More concretely:

\paragraph{Communication Lower Bounds without One-pass Assumption} 
Our current techniques for proving communication lower bounds seem 
to depend fairly strongly on the one-pass assumption. In particular, when 
bounding the information learned by the $t^{\textrm{th}}$ sample, it is critical 
for us to know that the information that we have from the current transcript 
is independent of that sample. Unfortunately, it is not clear how to get around that obstacle, 
and without it we have only the trivial lower bound of $\Omega(\sqrt{n}/(\eps^2 \ell))$.

\paragraph{Multi-pass Streaming Models} 
Another interesting open problem would be to consider multiple pass streaming models. 
For the reasons outlined above, it seems like our lower bounds would be difficult 
to generalize to even a two-pass streaming model. This leads to the interesting question 
of whether or not there are better algorithms in this model. At the very least, it is easy to 
see that the standard uniformity and closeness testers can be implemented with 
optimal sample complexity, $O(\log(n/\epsilon))$ memory, and $n$ passes over the data. 
What can be done with an intermediate number of passes?

\paragraph{Communication Lower Bounds for Closeness Testing}
We would like to show communication lower bounds for closeness testing that 
are not implied by our uniformity testing lower bounds and the general sample complexity lower bounds. 
Our current adversary method is  not sufficient for this task, as the testers that we have can 
distinguish our adversarial distributions from uniform in a small number of samples. 
In order to prove good closeness lower bounds, a more complicated adversary is necessary, 
and it is unclear how to combine this adversary with our information-theoretic arguments. 
It would even be interesting to make progress in this question for the case of constant $\epsilon$.

\paragraph{Extending Ranges of Validity} 
An immediate open question is to extend the range of validity of many of our bounds. 
Both our algorithms and lower bounds only work for constrained ranges of parameters 
in ways that do not allow us to adequately cover the whole space of parameters. 
It would be interesting to see if this dependence could be removed. 
Another interesting parameter range would be to see if there are any streaming algorithms 
at all with $o(\log(n))$ memory.

\paragraph{Instance-Optimal/Adaptive Testing} \cite{VV14} showed that testing identity 
to distributions other than the uniform distribution can often be done with better sample complexity 
in the centralized setting. It would be interesting to see what sort of analogue of this result 
can be obtained in our models. An analogous question can be asked for the adaptive
closeness tester of~\cite{DK16}.

\acks{I.D. is supported by NSF Award CCF-1652862 (CAREER) and a Sloan Research Fellowship.
This research was performed while T.G. was a postdoctoral researcher at USC, supported by I.D.'s
startup grant. D.K. is supported by NSF Award CCF-1553288 (CAREER) and a Sloan Research Fellowship. 
S.R. would like to thank Dheeraj P.V. for helpful discussions.}

\bibliography{allrefs}

\end{document}